%% file: main.tex
\documentclass{article}
\pdfoutput=1
\usepackage{ijcai23}

\usepackage{times}
\usepackage{soul}
\usepackage{url}
\usepackage[hidelinks]{hyperref}
\usepackage[utf8]{inputenc}
\usepackage[small]{caption}
\usepackage{graphicx}
\usepackage{amsmath}
\usepackage{amsthm}
\usepackage{booktabs}
\usepackage{algorithm}
\usepackage{algorithmic}
\usepackage[switch]{lineno}

\usepackage{balance} 
\usepackage{algorithm}
\usepackage{algorithmic}
\usepackage{color,amsfonts,amsmath}
\usepackage{bbm}
\usepackage{mathtools}
\usepackage{caption}
\usepackage{subcaption}
\usepackage{enumitem}
\usepackage{booktabs} 
\newcommand{\ra}[1]{\renewcommand{\arraystretch}{#1}}

\input{defs-mlmath}

\newcommand{\Y}{{X}}
\newcommand{\A}{{A}}

\newtheorem{theorem}{Theorem}
\newtheorem{lemma}{Lemma}

\title{Sequential Fair Resource Allocation under a  Markov Decision Process Framework}

\author{
Parisa Hassanzadeh$^1$\and
Eleonora Krea\v{c}i\'{c}$^1$\and
Sihan Zeng$^{2}$\and
Yuchen Xiao$^1$\And
Sumitra Ganesh$^1$
\affiliations
$^1$J.P. Morgan AI Research,\;
$^2$Georgia Institute of Technology
\emails
\{parisa.hassanzadeh, eleonora.kreacic, yuchen.xiao, sumitra.ganesh\}@jpmorgan.com,
szeng30@gatech.edu}

\begin{document}

\maketitle

\begin{abstract}
We study the sequential decision-making problem of allocating a divisible resource to agents that reveal their stochastic demands on arrival over a finite horizon. Our goal is to design fair allocation algorithms that exhaust the available resource budget. 
This is challenging in sequential settings where information on future demands is not available at the time of decision-making. We formulate the problem as a discrete time Markov decision process (MDP). We propose a new algorithm, SAFFE, that makes fair allocations with respect to all  demands revealed over the horizon by accounting for expected future demands at each arrival time. Using the MDP formulation, we show that the allocation made by SAFFE optimizes an upper bound of the Nash Social Welfare fairness objective. We further introduce SAFFE-D, which improves SAFFE by more carefully balancing the trade-off between the current revealed demands and the future potential demands based on the uncertainty in agents' future demands. We bound its gap to optimality using concentration bounds on total future demands. On synthetic and real data, we compare SAFFE-D against a set of baseline approaches, and show that it leads to more fair and efficient allocations and achieves close-to-optimal performance. 
\end{abstract}




\input{1_introduction}
\input{3_problem}
\input{4_offline}
\input{4_heuristic}

\input{4_optimality}
\input{5_experiments}

\input{2_related_work}
\input{6_conclusions}

\paragraph{Disclaimer}
This paper was prepared for informational purposes in part by
the Artificial Intelligence Research group of JPMorgan Chase \& Co\. and its affiliates (``JP Morgan''),
and is not a product of the Research Department of JP Morgan.
JP Morgan makes no representation and warranty whatsoever and disclaims all liability,
for the completeness, accuracy or reliability of the information contained herein.
This document is not intended as investment research or investment advice, or a recommendation,
offer or solicitation for the purchase or sale of any security, financial instrument, financial product or service,
or to be used in any way for evaluating the merits of participating in any transaction,
and shall not constitute a solicitation under any jurisdiction or to any person,
if such solicitation under such jurisdiction or to such person would be unlawful.

\bibliographystyle{named} 
\bibliography{references}

\newpage
\input{7_appendix}

\end{document}

%% file: defs-mlmath.tex
\newtheorem{remark}{Remark}

\usepackage{tikz}

%% file: 1_introduction.tex
\section{Introduction} \label{sec:intro}
The problem of multi-agent resource allocation arises in numerous disciplines including economics \cite{cohen1965theory}, power system optimization \cite{yi2016initialization,nair2018multi}, and cloud computing \cite{balasubramanian2007dtn,abid2020challenges}. It comes in various flavors depending on the nature of resources, the agents' behavior, the timing of allocation decisions, and the overall objective of the allocation process. A common element shared by various settings is the existence of a limited resource that needs to be distributed to various agents based on their reported requests.


There are various objectives that one might wish to optimize in the context of resource allocation. While the goal may be to maximize efficiency (e.g, minimize the leftover resource), other objectives combine efficiency with various notions of \textit{fairness}.
Among different fairness metrics,  Nash Social Welfare (NSW) is a well-known objective defined as the geometric mean of the agents' ``satisfaction'' with their allocation \cite{nash1950bargaining,kaneko1979nash}. In the \textit{offline} setting where we observe all requests prior to making an allocation decision,  the NSW solution enjoys favorable properties including being \textit{Pareto-efficient} (for any other allocation there would be at least one agent who is worse off compared to the current one), \textit{envy-free} (no agent prefers any other agent's allocation), and \textit{proportionally fair} (every agent gets their fair share of the resource) \cite{varian1973equity}, and thus exhausts the total budget under a limited resource.


In this work, we consider the much more challenging \textit{online} (or \textit{sequential}) setting where  agents place requests in a sequential manner, and allocation decisions have to be made instantaneously and are irrevocable. 
In order to achieve fairness we need to reserve the resource for anticipated future demands, which can lead to wasted resources if the expected future agents do not arrive. 
Unlike previous work on sequential allocation \cite{lien2014sequential,sinclair2022sequential}, we consider a more general setup where agents can arrive several times over the horizon, 
and their demand size is not restricted to a finite set of values. This setting has important applications
such as  allocating stock inventory on a trading platform or supplier-reseller settings (e.g., automaker-dealer) where the degree of fairness impacts their relationship  \cite{kumar1995effects,yilmaz2004supplier}.

\paragraph{Contributions}

We consider a general sequential setting where truthful agents can submit multiple demands for divisible resources over a fixed horizon, and the goal is to ensure fairness over this horizon. There is no standard notion of fairness in the sequential setting, as it may be impossible to satisfy all of the favorable fair properties simultaneously \cite{sinclair2022sequential}. Therefore, online allocation algorithms are commonly evaluated in comparison to their  offline (or \textit{hindsight}) counterpart \cite{gallego2015online,sinclair2022sequential}. We formulate the sequential problem in terms of maximizing expected NSW in hindsight, and design algorithms that explicitly optimize for this. 

The multiple-arrival setting requires being mindful of all past and future requests when responding to an agent's current demand. We propose a new algorithm SAFFE-D, which determines  fair allocations by accounting for all expected future \textit{unobserved} demands (e.g., based on historical information on agent demands) as well as each agent's past allocations. In addition, we introduce a tunable \textit{discounting} term in SAFFE-D, which improves both efficiency and fairness. It uses the uncertainty of future demands to balance the contention between allocating the resource now vs reserving it for the future across the entire  horizon.  
We provide theoretical upper bounds on the sub-optimality gap of SAFFE-D using concentration bounds around the expected future demands, which also guide us in tuning the discounting term of SAFFE-D. Finally, we use  numerical simulations and real data to illustrate the close-to-optimal performance of SAFFE-D in different settings, and we compare it against existing approaches and a reinforcement learning policy.

%% file: 3_problem.tex
\section{Problem Formulation} \label{sec:problem}
We consider a supplier that has a divisible resource with a limited budget size $B$, and $N$ truthful agents that arrive sequentially over $T$ time steps requesting (possibly fractional) units of the resource. At time $t\in \{1,\dots,T\}$, agent $i\in \{1,\dots,N\}$ arrives and reveals its demand $\Y^t_{i}\in \mathbb{R}_{\geq0}$ sampled from distribution $P_{X_i^t |\Y_i^{1},\dots, \Y_i^{t-1}}$. We assume that each agent has at least one demand over the horizon $T$. The supplier observes demands $\mathbf{\Y}^t=(\Y^t_{1},\dots, \Y^t_{N})$, and makes an  allocation $\mathbf{\A}^t=(\A^t_{1},\dots, \A^t_{N})$, where $\A^t_{i}\in \mathbb{R}_{\geq0}$ denotes the resource amount allocated to agent $i$. We assume that the allocations are immediately taken by the agents and removed from the inventory. We also assume that the setting is \textit{repeated}, i.e., after $T$ time steps, a new budget is drawn and the next allocation round of $T$ time steps starts. 

Agent $i$ has a utility function $u(\mathbf{\A}_i,\mathbf{X}_i)$, representing its satisfaction with the allocation $\mathbf{\A}_i=(\A_i^1,\dots,\A_i^T)$ given its (latent) demands. The utility function is a non-decreasing non-negative concave function.  In this work, we consider the following utility function, where an agent's utility linearly increases with its total allocation up to its total request  
\begin{align}
        u(\mathbf{\A}_i,\mathbf{X}_i)=\sum_{t=1}^T\min\left\{ \A_i^t, \,   \Y_i^t\right\}. \label{eq:utility}
    \end{align}
An agent with the utility function in \eqref{eq:utility} only values an allocation in the time step it requests the resource, and not in earlier or later steps, which is suitable in settings where the allocation process is time sensitive and the supplier is not able to delay the decision-making outside of the current time step. If all agent demands $\mathbf{X}_1,\dots, \mathbf{X}_N$ are known to the supplier at the time of decision-making, they can be used for determining the allocations. This setting is often referred to as  \textit{offline}. However, in the \textit{online} or \textit{sequential} setting, the agent demands are gradually revealed over time, such that the supplier only learns about $X_1^t,\dots, X_N^t$ at time $t$.  We present the paper in terms of a single resource; however, the setting  extends to  multiple resources as described in Appendix~\ref{app:multi-resource}.

 \paragraph{\textbf{Notation}} 
$(x)^+ = \max\{x,0\}$. For vectors $\mathbf{X}$ and $\mathbf{Y}$, we use $\mathbf{X}\geq\mathbf{Y}$ to denote $X_i\geq Y_i$ for each  $i$. $\mathbf{0}$ denotes a zero vector. $N(\mu, \sigma^2)$ denotes the Normal distribution with mean $\mu$ and variance $\sigma^2$.

\subsection{Fairness in Allocation Problems}
The supplier aims to allocate the resource efficiently, i.e., with minimal leftover at the end of horizon $T$, and in a fair manner maximizing  NSW. The NSW objective is defined as $\prod_{i=1}^{N} u(\A_i,\mathbf{X}_i)^{w_i}$, where $w_i\in\mathbb{R}_{+}$ reflects the weight assigned to agent $i$  by the supplier. NSW is a balanced compromise between the utilitarian welfare objective, which maximizes the utility of a larger number of agents, and the egalitarian welfare objective, which maximizes the utility of the worst-off agent. Since  
\begin{align}
     \arg&\max_{\mathbf{A}_i}  \prod_{i=1}^N  u(\mathbf{A}_i,\mathbf{X}_i)^{w_i}   = \arg\max_{\mathbf{A}_i}  \sum_{i=1}^N w_i\log  u(\mathbf{A}_i,\mathbf{X}_i) ,\notag
\end{align}
the logarithm of NSW is often used as the objective function in allocation problems, such as in the Eisenberg-Gale program \cite{eisenberg1959consensus}. We refer to this as the log-NSW objective.
 
In the sequential setting, it is not guaranteed that there always exists an allocation that is simultaneously  Pareto-efficient, envy-free, and proportional \cite[Lemma 2.3]{sinclair2020sequential}. Motivated by the properties of the NSW objective in the offline setting, and the fact that our repeated setting can be viewed as multiple draws of the offline setting, we use a modified NSW objective for the sequential setting. Our goal in the sequential setting is to find an allocation that maximizes the log-NSW in expectation. Specifically, 
 \begin{align}
 \arg\max_{\mathbf{A}_i}  \sum_{i=1}^N w_i\mathbb{E}_{\mathbf{X}_i} \Big[\log u(\mathbf{A}_i,\mathbf{X}_i)\Big]. \label{eq:nsw hindsight}
 \end{align}

With the goal of measuring the performance of an allocation in the sequential setting, \cite{sinclair2020sequential} introduces approximate fairness metrics $\Delta \mathbf{A}^\text{max}$ and $\Delta \mathbf{A}^\text{mean}$.
Let $\mathbf{A}_i^{\text{online}}$ denote the allocation vector of agent $i$ given by an online algorithm subject to latent demands, and let  $\mathbf{A}_i^{\text{hindsight}}$ denote the allocation vector in hindsight after all  demands are revealed, defined as \eqref{eq:hindsightA} in Sec.~\ref{sec:offline}. 
The expected difference between the overall allocations for each agent can be used to measure the fairness of the online algorithm. 
Let $\Delta A_i\coloneqq\Big| \sum_{t=1}^{T}{A}_i^{t,\text{hindsight}}-\sum_{t=1}^{T}{A}_i^{t,\text{online}} \Big|$. Then,
\begin{align}
\Delta \mathbf{A}^\text{max}= \mathbb{E}\Big[\max_{i} \Delta A_i\Big], \; \Delta \mathbf{A}^\text{mean}=\frac{1}{N}\sum_{i=1}^N \mathbb{E}[  \Delta A_i], \label{eq:deltaA}
\end{align} 
are concerned with the worst-off agent and average agent in terms of cumulative allocations in hindsight, respectively\footnote{$\Delta \mathbf{A}$ is defined for the additive utility function in~\eqref{eq:utility}, and may need to be redefined for alternative utilities.}. While our optimization objective is not to minimize $\Delta \mathbf{A}^\text{max}$ or $\Delta \mathbf{A}^\text{mean}$, we use these  metrics to evaluate the fairness of online allocation algorithms in the experiments in Sec.~\ref{sec:experiments}.

\subsection{Markov Decision Process Formulation}\label{sec:mdp}
Determining the optimal allocations under the NSW objective is  a sequential decision-making problem as the allocation in one step affects the budget and allocations in future steps. We formulate the problem as a finite-horizon total-reward Markov Decision Process (MDP) modeled as a tuple  $<\{\mathcal{S}_t,\mathcal{A}_t,P_t,R_t\}_{t=1,\dots,T}>$.  $\mathcal{S}_t$ denotes the underlying time-dependent state space, $ \mathcal{A}_t$ is the action space, $P_t : \mathcal{S}_t \times \mathcal{A}_t\times \mathcal{S}_{t+1} 
\rightarrow \mathbb{R}_{\geq0}$ describes the state transition dynamics conditioned upon the previous state and action, 
$R_t : \mathcal{S}_t \times \mathcal{A}_t
\rightarrow \mathbb{R}_{\geq0}$ is
a non-negative reward function, and $T$ is the horizon over which the resource is allocated. The goal of the supplier is to find an allocation  policy 
$\pi=\{\pi_1,\dots,\pi_T\mid\pi_t:  \mathcal{S}_t\rightarrow \mathcal{A}_t\}$
mapping the state to an action, in order to maximize the expected sum of rewards $\mathbb{E}[\sum_{t=1}^T R_t(s_t, \pi_t(s_t))]$. Next, we describe the components of the MDP in details. 
\paragraph{\textbf{State Space}}
 The state space $\mathcal{S}_t$ is time-dependent, and the state size increases with time step $t$ since the state captures the past demand
and allocation information. Specifically, the state at step $t$ is defined as $s_t = (\mathbf{\Y}^{1:t},\, \mathbf{\A}^{1:t-1},\,B^t)$, where $\mathbf{\Y}^{1:t}\coloneqq (\mathbf{\Y}^1,\dots, \mathbf{\Y}^t)$, $\mathbf{\A}^{1:t}\coloneqq (\mathbf{\A}^1,\dots, \mathbf{\A}^t)$, and 
\begin{align}
B^t = \begin{cases}
B^{t-1} -\sum_{i=1}^N A_{i}^{t-1} &\quad t\geq 1 \\
 B &\quad t=1
\end{cases}    
\end{align}

\paragraph{\textbf{Action Space}}
The actions space is state and time-dependent. For any $s_t=(\mathbf{\Y}^{1:t},\, \mathbf{\A}^{1:t-1},\,B^t)\in \mathcal{S}_t$, we have
\begin{align}
    \mathcal{A}_t = \{\mathbf{A}^t \in \mathbb{R}^{N}:\; \sum_{i=1}^N A^t_i\leq B^t\}\label{eq:action-space}
\end{align}
The state and action space are both continuous and therefore infinitely large. However, $\mathcal{A}_t$ is a compact polytope for any $s_t\in\mathcal{S}_t$, and $\mathcal{S}_t$ is compact if the requests are bounded.

\paragraph{\textbf{State Transition Function}}

Given state $s\in \mathcal{S}_t$ and action $a\in\mathcal{A}_t$, the system transitions to the next state $s'\in \mathcal{S}_{t+1}$ with probability \begin{align}
P(s,a,s')=\text{Prob}(s_{t+1}=s'\mid s_{t}=s,a_{t}=a),
\end{align}
where $X_i^{t+1}\sim P_{X_i^{t+1}|X_i^1,\dots,X_i^t}$.

\paragraph{\textbf{Reward Function}}
The reward at time step $t\in\{1,\dots,T\}$, is defined as follows
\begin{align}
    R_t(s_t, \pi_t(s_t)) =& \sum_{i=1}^N  \mathbbm{1} \{X_i^t >0\} \;.\; w_i (U_i^t - U_i^{t-1})    , \label{eq:step-reward}
\end{align}
where
\begin{align}
U_i^t &=
  \log \left(\sum_{\tau=1}^t \min\{A_i^{\tau}, X_i^{\tau}\} + \epsilon \right), \; t\in\{1,\dots,T\},\label{eq:def_U}
\end{align}
$U_i^{0}=0$, and $\epsilon$ is a small value added for to ensure values are within the domain of the log function\footnote{For $\epsilon\ll1$, this will lead to a maximum error of $\epsilon$ for $x\geq 1.$}. The reward $R_t$ denotes the weighted sum of incremental increases in each agent's utility at time $t$\footnote{For agent $i$, we have $\sum_{t=1}^T \mathbbm{1} \{X_i^t >0\} (U_i^t - U_i^{t-1})  =  \sum_{t=1}^T(U_i^t - U_i^{t-1}) = U_i^T = \log (\sum_{t=1}^T \min\{A_i^{t}, X_i^{t}\} + \epsilon )$, which matches the agent's log-utility as $\epsilon\rightarrow 0$.}. The indicator function $  \mathbbm{1} \{X_i^t >0\}$ ensures that we only account for the agents with a demand at time $t$. Then, the  expected sum of rewards over the entire horizon $T$ is equivalent to the expected log-NSW objective defined in \eqref{eq:nsw hindsight} for $\epsilon\rightarrow0$.

At time $t$, with state $s_t\in \mathcal{S}_t$ and action (allocation) $\mathbf{A}^t\in \mathcal{A}_t$, the optimal Q-values satisfy the  Bellman optimality equation \cite{bellman1966dynamic}:
\begin{align}
    Q_t(s_t,\mathbf{A}^t)&= 
    R_t(s_t, \mathbf{A}^t)
    +\mathbb{E}\Big[\max_{\mathbf{A}^{t+1} \in \mathcal{A}_{t+1}}Q_{t+1}(s_{t+1},\mathbf{A}^{t+1})\Big].\label{eq:bellman}
\end{align}
We denote the optimal policy corresponding to Eq.~\eqref{eq:bellman} by  $\pi_t^\star$. Then, the optimal allocation is the solution to
\begin{align}
    \arg\max_{\mathbf{A}^{t}} Q_t(s_t,\mathbf{A}^t). \label{eq:optalg_timet}
\end{align}

 We  show in Appendix~\ref{app:existance} that an optimal policy exists for the MDP, and it can be derived by recursively solving the  Bellman optimality equation \cite{bellman1966dynamic} backward in time. This quickly becomes computationally intractable, which motivates us to find alternative solutions that trade-off sub-optimality for computational efficiency. In Sec.~\ref{sec:heuristic}, we introduce an algorithm that uses estimates of future demands to determine the allocations, and we discuss a learning-based approach in Sec.~\ref{sec:exp-RL}.

%% file: 4_offline.tex
\vspace{-5pt}
\section{Offline (Hindsight) Setting} \label{sec:offline}
If the supplier could postpone the decision-making till the end of the horizon $T$, it would have perfect knowledge of all demands $\mathbf{X}^1,\dots, \mathbf{X}^T$. Let $\widetilde{\Y}_i = \sum_{\tau=1}^T X_i^\tau$ denote the total demands of agent $i$. Then, the supplier solves the following convex program, referred to as the Eisenberg-Gale program \cite{eisenberg1961aggregation}, to maximize the  log-NSW objective in \eqref{eq:nsw hindsight} for allocations $\widetilde{\mathbf{\A}}=(\widetilde{\A}_1,\dots,\widetilde{\A}_N)$, 
\begin{equation}
\begin{aligned}
    \max_{\widetilde{\mathbf{A}}\geq \mathbf{0} }  \sum_{i=1}^N  w_i\log \Big( u(\widetilde{A}_i,\widetilde{X}_i)\Big) \quad
    \text{s.t. }
    &\sum_{i=1}^{N}\widetilde{A}_i\leq B. 
\end{aligned}\label{eq:obj_offline0}
\end{equation}
While the optimal solution to \eqref{eq:obj_offline0} may not be achievable in  the sequential setting, it provides an upper bound on the log-NSW achieved by any online algorithm, and serves as a baseline for comparison. With the utility function in~\eqref{eq:utility}, allocating resources beyond the agent's request does not increase the utility. Therefore, solving \eqref{eq:obj_offline0} is equivalent to the following
\begin{equation} 
\begin{aligned}
    \max_{\widetilde{\mathbf{A}} } \;& \sum_{i=1}^N  w_i\log (\widetilde{A}_i) \\
    \text{s.t. }
    & 0 \leq \widetilde{A}_i\leq \widetilde{X}_i, \; \forall i \quad
    \sum_{i=1}^{N}\widetilde{A}_i\leq B. 
\end{aligned}\label{eq:obj_offline}
\end{equation}
Then, any  distribution of agent $i$'s allocation $\widetilde{A}_i$ across the $T$ time steps that satisfies the demand constraint at each step, would be an optimal allocation \textit{in hindsight} $\mathbf{A}_i^\text{hindsight}$, i.e.,
\begin{equation}
\mathbf{A}_i^\text{hindsight} \in \Big\{ (A_i^1,\dots, A_i^T):0\leq A_i^t\leq X_i^t, \forall t, \; \sum_{\tau=1}^T A_i^t = \widetilde{A}_i   \Big\}.\label{eq:hindsightA}
\end{equation}
The optimal solution to \eqref{eq:obj_offline} takes the form $\widetilde{A}_i^{\star}=\min\{\widetilde{X}_i,\mu\}$, where 
$\mu$ is a function of budget $B$ and demands $\widetilde{\mathbf{X}}$ such that $\sum_{i=1}^N \widetilde{A}_i=B$. The solution can be efficiently derived by the water-filling algorithm in Algorithm~\ref{alg:waterfilling-base},  and the threshold $\mu$ can be interpreted as the water-level shown in Fig.~\ref{fig:waterfill}.
\begin{figure}[!h]
\centering
\begin{subfigure}{0.47\linewidth} \centering
 \includegraphics[width=\linewidth]{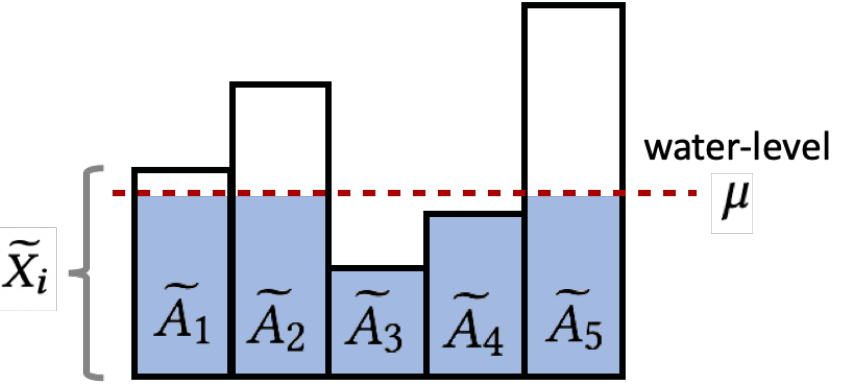}
\caption{ }
\label{fig:waterfill}
\end{subfigure}%
%
\begin{subfigure}{0.47\linewidth} \centering
\includegraphics[width=\linewidth]{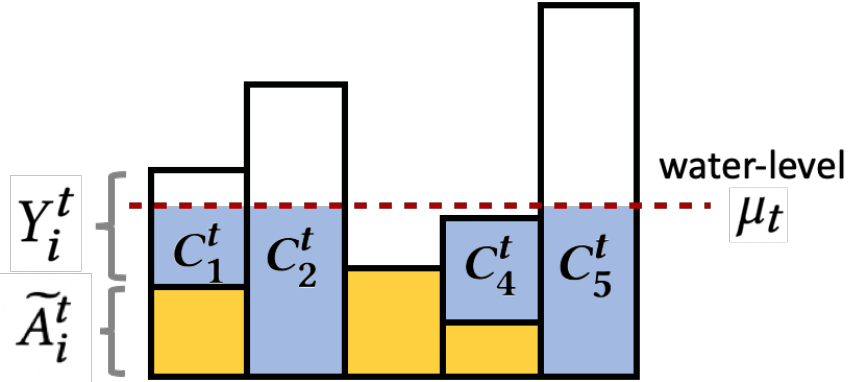}
\caption{ }
\label{fig:waterfill-modified}
\end{subfigure}%
\vspace{-3mm}
\caption{Water-filling solution for equal unit weights: Each bar shows the demands or allocations of an agent. (a) In hindsight, the resource is allocated using  water-level $\mu$ determined such that the total allocations are equal to the budget size. (b) In the sequential setting, at time $t$, the cumulative prior allocations and total expected future demands are accounted for when determining  water-level $\mu_t$.}
\label{fig:three graphs}
\end{figure}
\begin{algorithm}
\caption{\small Water-filling algorithm with weights for solving \eqref{eq:obj_offline}}
\label{alg:waterfilling-base}
\begin{algorithmic}[1]
\STATE{Input: number of agents $N$, resource budget $B$, demand vector $\widetilde{\mathbf{X}}\in\mathbb{R}^N$, weight vector $\mathbf{w}\in\mathbb{R}^N$
}
\STATE{Output: allocation vector $\widetilde{\mathbf{A}}\in\mathbb{R}^N$}
\STATE{Find an ordered index set $\{i_1,\dots,i_N\}$ by sorting the agents such that $\frac{\widetilde{X}_{i_1}}{w_{i_1}}\leq\dots\leq \frac{\widetilde{X}_{i_N}}{w_{i_N}}$}
\STATE{$\alpha_j = \frac{w_{i_j}}{w_{i_j}+\dots+w_{i_N}}$ for $j=1,\dots, N$}
\STATE{$j \leftarrow 1$}
\WHILE{$j\leq N$ and $B > 0$}
\IF{$B \leq \widetilde{X}_{i_j}/\alpha_j $ }
     \STATE{$\widetilde{A}_{i_k}=\alpha_k B$ for $k=j,\dots, N$}
  \STATE{$\textbf{break}$}
\ELSE
  \STATE{$\widetilde{A}_{i_j}=\widetilde{X}_{i_j}$}  
  \STATE{$B \leftarrow B - \widetilde{A}_{i_j}$}
  \STATE{$j\leftarrow j+1$} 
\ENDIF
\ENDWHILE
\end{algorithmic}
\end{algorithm}

%% file: 4_heuristic.tex
\section{Sequential Setting: Heuristic Algorithm} \label{sec:heuristic}
In this section, we propose an intuitive and computationally efficient  algorithm, named \textit{Sequential Allocations with Fairness based on Future Estimates} (SAFFE), and its variant SAFFE-Discounted. This online algorithm uses Algorithm~\ref{alg:waterfilling-base} as a sub-routine to compute the allocations at each step $t$.

\subsection{SAFFE Algorithm}\label{subsec:saffe}
SAFFE is based on the simple principle of 
substituting unobserved future requests with their expectations.
By doing so, we convert the sequential decision-making problem into solving the offline problem at each time step. With the algorithm formally stated in Algorithm~\ref{alg:SAFFE}, at  time step $t$ we use the expected future demands to determine the total resources we expect to allocate to each agent by time $T$. This allows us to reserve a portion of the available budget for future demands. Specifically, at $t=1,\dots, T$, we solve the following problem
\begin{equation} 
\begin{aligned}
    \max_{\mathbf{C}^t} \;& \sum_{i=1}^N \mathbbm{1} \{Y_i^t >0\} \;.\; w_i\log (\widetilde{A}_i^t + C_i^t) \\
    \text{s.t. }\;
    & 0 \leq C_i^t\leq Y_i^t , \; \forall i, \quad \sum_{i=1}^{N}C_i^t\leq B^t,
\end{aligned}\label{eq:obj_online}
\end{equation}

\vspace{-5mm}
\begin{align}
\text{where \qquad} Y_i^t = X_i^t+\sum_{\tau=t+1}^{T} \mathbb{E}[X_i^{\tau}], \; i=1,\dots, N.  \label{eq:demand+future} 
\end{align}
The indicator function $  \mathbbm{1} \{{Y}_i^t >0\}$ ensures that we do not consider absent agents, i.e., those with no current or expected future demands\footnote{With a slight abuse of notation, in the objective function, we assume  $0\times -\infty = 0$ to ensure that an agent without demand is not included, and gets no allocation.}. ${C}_i^t$ denotes the total allocation for agent $i$ over the period $t,\dots,T$, if the future demands would  arrive exactly as their expectations. In other words, ${C}_i^t$ consists of the allocation in the current time step and the reserved allocations for the future. Then, the current allocation $A_i^t$ is a fraction of $C_i^t$ computed as
\begin{align}
    A_i^t=C_i^t\frac{X_i^t }{Y_i^t}.\label{eq:split}
\end{align}
Similar to the hindsight problem \eqref{eq:obj_offline}, we can efficiently solve  \eqref{eq:obj_online} using a variant of the water-filling algorithm given in  Algorithm~\ref{alg:waterfilling past alloc} in Appendix~\ref{app:waterfilling-past alloc}. As illustrated in Fig.~\ref{fig:waterfill-modified}, at  time step $t$, the water-level $\mu_t$ is calculated while accounting for all previous agent allocations and all expected future demands.

\begin{algorithm}
\caption{SAFFE Algorithm}
\label{alg:SAFFE}
\begin{algorithmic}[1]
\STATE{Input: number of agents $N$, resource budget $B$, demand vectors $\mathbf{X}^1,\dots \mathbf{X}^T\in\mathbb{R}^N$, weight vector $\mathbf{w}\in\mathbb{R}^N$, demand distributions  $P_{\mathbf{X}_1},\dots,P_{\mathbf{X}_N}$}
\STATE{Output: allocation vectors $\mathbf{A}^1,\dots,\mathbf{A}^T\in\mathbb{R}^N$}
\FOR{For $t=1,\dots,T$}
\STATE{$Y_i^t= X_i^t+\mathbb{E}[\sum_{\tau=t+1}^{T}X_i^{\tau}]$ for $i=1,\dots,N$}
\STATE{$\widetilde{A}^t_i= \sum_{\tau=1}^{t}A_i^{\tau}$ for $i=1,\dots,N$}
\STATE{$\mathbf{C}^t \leftarrow $ Algorithm~\ref{alg:waterfilling past alloc} with input ($N$, $B$, $\mathbf{Y}^t, \mathbf{w}, \widetilde{\mathbf{A}}^t$) in Appendix~\ref{app:waterfilling-past alloc}}
\STATE{$A_i^t=C_i^t  X_i^t/Y_i^t$ for $i=1,\dots,N$}
\STATE{$B \leftarrow B-\sum_{i=1}^NA_i^t$}
\ENDFOR
 \end{algorithmic}
\end{algorithm}

\vspace{-1em}
\subsection{SAFFE-Discounted Algorithm}\label{sec:saffe-D}
While SAFFE is simple and easy to interpret, it is sub-optimal. In fact, SAFFE maximizes an upper bound of the optimal Q-values of the MDP defined in Sec.~\ref{sec:mdp}. 
This implies that SAFFE overestimates the expected reward gain from the future and reserves the budget excessively for future demands.  To correct the over-reservation, we propose SAFFE-Discounted (SAFFE-D), which penalizes uncertain future requests by their standard deviations. At every step $t$, the algorithm computes
\begin{align}
Y_i^t = \Y_i^t+\sum_{\tau=t+1}^{T}\Big(\mathbb{E}[\Y_i^{\tau}]-\lambda\operatorname{std}(\Y_i^{\tau})\Big)^+
\end{align}
for some regularization parameter $\lambda\geq0$ and solves \eqref{eq:obj_online}. As in SAFFE, the current allocation $A_i^t$ is split proportionally from $C_i^t$ according to  \eqref{eq:split}. The regularizer $\lambda$ is a hyper-parameter that can be tuned to provide the optimal trade-off between consuming too much of the budget at the current step and reserving too much for the future. We show in Appendix \ref{app:SAFFE-D Vs SAFFE-O}, the uncertainty in future demands reduces as we approach $T$, and we expect better performance with a decreasing time-dependent function such as $\lambda(t)=\sqrt{T-t}\,\lambda$ for some $\lambda>0$. Alternatively, $\lambda(t)$ can be learned as a function of time. 

\begin{remark}
 In this paper, we assume access to historical data from which expected future demands can be estimated, and we mainly focus on the decision-making aspect for allocations. These estimates are directly used by SAFFE and SAFFE-D. We empirically study how sensitive the algorithms are to estimation errors in Sec.~\ref{sec:exp-sensitivty}.
\end{remark}

%% file: 4_optimality.tex
\section{On the Sub-Optimality of SAFFE-D}\label{sec:online-optimality}
In order to determine how sub-optimal SAFFE-D is, we upper bound the performance gap between SAFFE-D and the hindsight solution, in terms of $\Delta\mathbf{A}^\text{max}$ defined in  \eqref{eq:deltaA}. Based on  \cite[Lemma 2.7]{sinclair2020sequential}, this bound translates to bounds on Pareto-efficiency, envy-freeness, and proportionality of the same order by Lipschitz continuity arguments, which extend to 
SAFFE-D's solution. To this end, we consider two demand settings: a worst-case scenario where  we allow highly unbalanced  demands for different agents, and a balanced-demand scenario where we assume  equal expectation and variance across various agents (see Appendix \ref{app:SAFFE-D Vs SAFFE-O} for details). In the worst-case scenario, we assume that there is one agent who has much larger requests compared to all other agents, and thus, an increase in the request of any other agent would be on her account. This is a fairly general case - the only assumption we make is knowing the mean and variance of future demand distributions. Our arguments rely on concentration inequalities and are in spirit similar  to the idea of \textit{guardrails} in  \cite{sinclair2022sequential}. The following theorem provides the gap between hindsight and SAFFE-D solutions for $\lambda(t)=\sqrt{(T-t)/\xi}$. For brevity we impose the assumption of equal variance $\operatorname{std}(X^{t}_{i})$ across all agents, which can be relaxed for more general settings as in  Remark \ref{remark: relax same std assumption by agents} of Appendix \ref{app:proof of SAFFE-D Vs SAFFE-O}.

\begin{theorem}[Gap between SAFFE-D and Hindsight]\label{thm:final gap} 
With probability at least $1-\xi$, the gap between SAFFE-D and hindsight measured by $\Delta\mathbf{A}^\text{max}$ introduced in  \eqref{eq:deltaA} satisfies
\begin{align}
 \Delta\mathbf{A}^\text{max} \leq \frac{2 T^{3/2}}{\sqrt{\xi}} \operatorname{std}(X^{t}_{i}),
\end{align}
for balanced demands. In the worst-case scenario this bound is $N\frac{T^{3/2}}{\sqrt{\xi}} \operatorname{std}(X^{t}_{i})$, and scales with the number of agents $N$.
\end{theorem}
 \begin{proof} The detailed proof is in Appendix \ref{app:detailed proof final thm}. Let $\mathbf{A}^\text{oracle}$ denote the allocations derived using SAFFE-D in a setting where an \textit{oracle} has perfect knowledge of the sequence of incoming demands, i.e., 
 with no stochasticity. We refer to this setup as {SAFEE-Oracle}. We upper bound $\Delta\mathbf{A}^\text{max}$, the distance between $\mathbf{A}^\text{SAFFE-D}$ and $\mathbf{A}^\text{hindsight}$, in two steps: by bounding the distance between SAFFE-D and SAFFE-Oracle (Appendix~\ref{app:SAFFE-D Vs SAFFE-O}) and  between SAFFE-Oracle and hindsight (Appendix~\ref{app: discrepancy saffe oracle achieves hindsight}). Finally, triangle inequality completes the proof and upper bounds the distance $\Delta\mathbf{A}^\text{max}$ between SAFFE-D and hindsight in terms of the sum of these two distances. 
\end{proof}

%% file: 5_experiments.tex
\section{Experimental Results} \label{sec:experiments}

In this section, we use synthetic data to evaluate SAFFE-D  under different settings and compare its performance against baseline sequential algorithms in terms of fairness metrics and budget allocation efficiency. In the experiments, we study: 1) how the fairness of SAFFE-D allocations is affected as the budget size, number of agents and time horizon vary (Sec.~\ref{sec:exp-scaling}), 2) whether the algorithm favors agents depending on their arrival times or demand sizes (Sec.~\ref{sec:exp-NH}), 3) how sensitive the algorithm is to future demand estimation errors  (Sec.~\ref{sec:exp-sensitivty}), 4) how the discounting in SAFFE-D improves fairness (Sec.~\ref{sec:exp-lambda}), and 5) how SAFFE-D compares to allocation policies learned using reinforcement learning (RL) (Sec.~\ref{sec:exp-RL}). Finally, in Sec.~\ref{sec:exp-real}, we evaluate SAFFE-D on real data.

{\setlength{\parindent}{0cm}\paragraph{\textbf{Evaluation Metrics}} We consider the following metrics to compare the fairness and efficiency of sequential allocation algorithms:}
\begin{itemize}[leftmargin=1.5em]
    \item \textbf{Log-NSW}: Expected log-NSW in hindsight as in Eq. \eqref{eq:nsw hindsight}. Since the value of log-NSW is not directly comparable across different settings, we use its normalized distance to the hindsight log-NSW when comparing algorithms,  denoted by $\Delta\text{Log-NSW}$.
    \item \textbf{Utilization (\%)}: The fraction of available budget distributed to the agents over the horizon. Due to the stochasticity in experiments described next, we may have an over-abundance of initial budget that exceeds the total demands over the horizon. Therefore, we define this metric by only considering the required fraction of available budget $B$ as 
    \begin{align}
     \frac{\sum_{i=1}^N\sum_{t=1}^T A_i^t}{\min\{B, \sum_{i=1}^N\sum_{t=1}^T X_i^t\}}\times 100 \notag
    \end{align}
    \item \textbf{$\Delta \mathbf{A}^\text{mean}$ and $\Delta \mathbf{A}^\text{max}$}: The average and maximum \textit{normalized} deviation of per-agent cumulative allocations compared to hindsight allocations as defined in Eq. \eqref{eq:deltaA}. For better scaling in  visualizations, we make a slight change by normalizing these metrics with respect to hindsight as $\frac{\Delta A_{i}}{\sum_{t=1}^T A_i^{t,\text{hindsight}}}$. Since these metrics measure distance, an algorithm with lower $\Delta\mathbf{A}$ is considered more fair in terms of this metric.
\end{itemize}

 {\setlength{\parindent}{0cm}\paragraph{\textbf{Allocation Algorithm Baselines}} 
We compare our algorithms SAFFE and SAFFE-D with the following baselines:} 
\begin{itemize}[leftmargin=1.5em]
    \item \textbf{Hindsight}: As discussed in Sec.~\ref{sec:offline}, the solution to (\ref{eq:obj_offline}) represents a baseline for evaluating sequential algorithms since its Log-NSW provides an upper bound for other  algorithms.
    \item \textbf{HOPE-Online}: Motivated by the mobile food-bank allocation problem, this algorithm was proposed in \cite{sinclair2020sequential} for a setting where $N$ agent demands are sequentially revealed over $T=N$. As we explain in Appendix~\ref{app:SAFFE-hope}, HOPE-Online coincides with SAFFE in the special case where each agent makes only one request.
    \item \textbf{Guarded-HOPE}: It was proposed in \cite{sinclair2022sequential} to achieve the correct trade-off between fairness (in terms of envy-freeness) and efficiency with an input parameter $L_T$ derived based on concentration arguments. The algorithm is designed for a setting of one request per horizon for an individual. We modify Guarded-HOPE (Algorithm \ref{alg:GuardedHOPE} in Appendix~\ref{app:gaurded-hope}) to be applicable to our setting, and  to provide meaningful confidence bounds for the multiple demand case. As in \cite{sinclair2022sequential}, we use $L_T = T^{-1/2}$ and $L_T = T^{-1/3}$. 
\end{itemize}

 {\setlength{\parindent}{0cm}\paragraph{\textbf{Demand Processes}} We investigate various settings by considering the following demand processes, for which we are able to analytically derive the future demand means and standard deviations:} 
\begin{itemize}[leftmargin=1.5em]
    \item \textbf{Symmetric Setting} (Homogeneous Bernoulli arrivals with Normal demands). At time $t$, agent $i$ makes a request $X_i\sim N(\mu_{i}, \sigma_{i}^2)$ with probability $p$, independently from other agents. The distribution parameters are such that $\mu_i\sim \text{Uniform}(10, 100)$, and  $\sigma_{i}=\mu_{i}/5$.  We study regimes where $p=\frac{c}{T}$, $1\leq c \leq T$, is the average number of arrivals per agent. 
    \item \textbf{Non-symmetric Arrivals Setting} (Inhomogeneous Bernoulli arrivals with Normal demands). In this setup, we assume that the likelihood of arrivals over the horizon  varies across agents, which we use to evaluate whether the algorithms favor agents based on their arrival times. We implement this using Bernoulli arrivals with a time-varying parameter $p$. We consider three groups. \textit{Ask-Early}:  $1/3^\text{rd}$ of the agents are more likely to frequently visit earlier in the horizon such that $p^{t}_{i} \propto (T-t)$, \textit{Ask-Late}: $1/3^\text{rd}$ of the agents are more likely to frequently visit later in the horizon with $p^{t}_{i} \propto t$, and \textit{Uniform}: the remaining agents do not have a preference in terms of time with constant $p^{t}_{i}$. 
    The Bernoulli parameters of the three groups are normalized to have equal expected number of visits.
    \item \textbf{Non-symmetric Demands Setting} (Homogeneous Bernoulli arrivals and varying Normal demands). 
    We consider homogeneous arrivals but with time-varying demand sizes across agents. We have three groups with agents that make requests with the same probability $p$ over $T$. \textit{More-Early}: $1/3^\text{rd}$ of the agents have Normal demands with mean $\mu^{t}_{i} \propto (T-t)$, \textit{More-Late}: $1/3^\text{rd}$ of the agents have $\mu_{i}^t\propto t$, and \textit{Uniform}: the remaining agents have constant $\mu_{i}^t$. 
    The parameters are normalized to have the same total demands across the groups. Standard deviation $\sigma_i$ is assumed to be  time-independent for all agents.
\end{itemize}

{\setlength{\parindent}{0cm}\paragraph{\textbf{Real Data}} We evaluate SAFFE-D using a  real dataset\footnote{\url{https://www.kaggle.com/c/demand-forecasting-kernels-only}} containing five years of store-item sales data for $10$ stores. We assume that each store (agent) requests its total daily sales  from a warehouse (supplier) with limited inventory, and the warehouse aims to be fair when supplying the stores on a weekly basis, i.e., when $T=7$. We use the first three years to estimate the expected demands and standard deviations for each weekday, and we use the remaining two years to evaluate SAFFE-D for fair daily allocations.
}

\vspace{0.2cm}
We report the results on synthetic data as an average over $200$ experiment realizations with one standard deviation. For SAFFE-D we use $\lambda(t)=\lambda\sqrt{T-t}$, where the hyper-parameter $\lambda$ is optimized with respect to Log-NSW.  We express the  budget size $B$ as the fraction of total expected demands such that a budget of $0.1$ means enough budget to meet only $10\%$ of the total expected demands. We assume equal weights across all agents.

\vspace{-5pt}
\subsection{Scaling System Parameters}\label{sec:exp-scaling}
In Fig.~\ref{fig:scale-B-main}, we compare SAFEE-D with the baselines as the supplier's budget increases from $0.1$ to $1$. 
We consider the Symmetric setting with $N=50$ agents and $2$ expected arrivals over horizon $T=40$. We observe that SAFFE-D outperforms all other approaches in terms of both $\Delta$Log-NSW and resource utilization across all budgets. The advantage of SAFFE-D in comparison to other algorithms is that it
prevents over-reserving for the future, and thus achieves higher utilization and lower $\Delta$Log-NSW. In contrast, SAFFE does not impose the regularization term which potentially results in overestimating future demands. Hope-Guardrail is designed to find a balance between utilization and envy, which does not necessarily mean balance in terms of Log-NSW. Additional experiments showing how the performance of the algorithms scale with the number of agents $N$, time horizon $T$ and the  per-agent expected arrivals are provided in Appendix~\ref{app:exp-scaling}. The results suggest that SAFFE-D, although suboptimal, performs close to  Hindsight.

\begin{figure*}[!h]
\centering
\begin{subfigure}{\linewidth} \centering
 \includegraphics[width=0.93\linewidth]{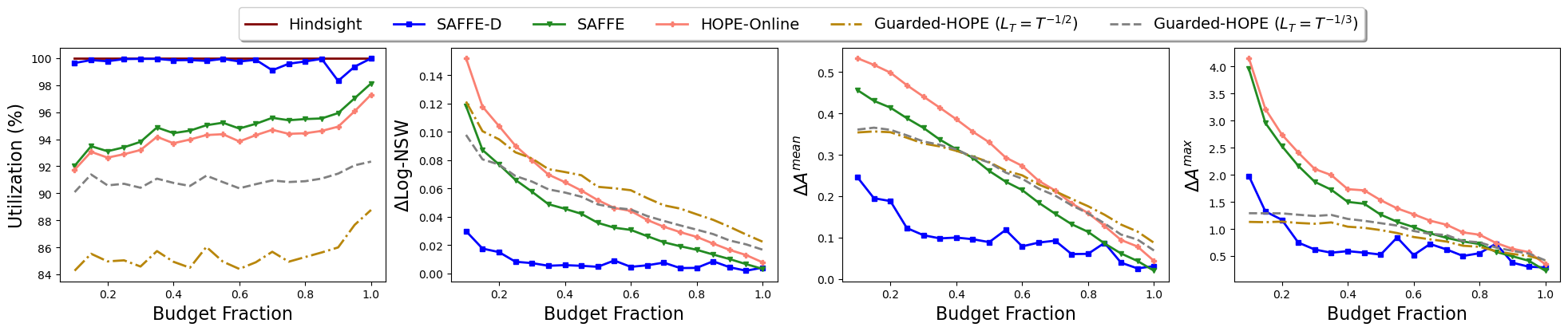}
\caption{Different budget size $B$.}
\label{fig:scale-B-main}
\end{subfigure}%
\\
\begin{subfigure}{\linewidth} \centering
\includegraphics[width=0.91\linewidth]{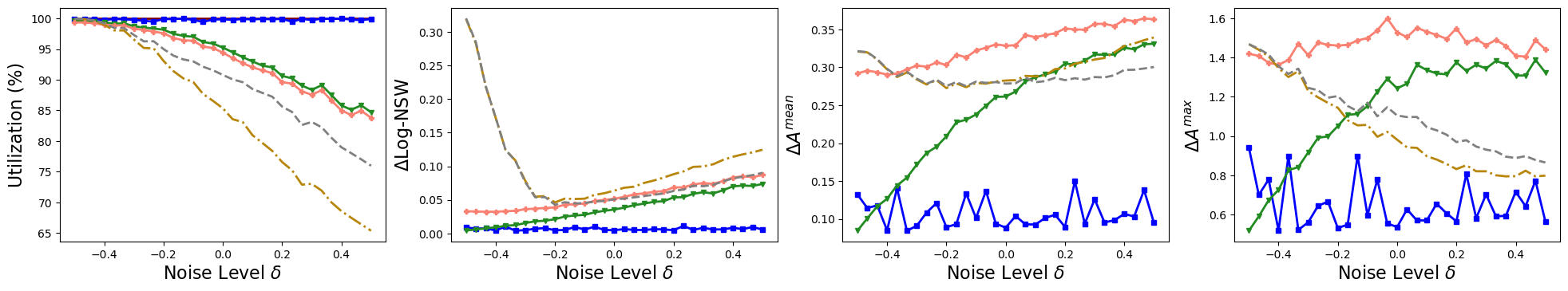}
\caption{Sensitivity to estimation errors.}
\label{fig:sensitivity}
\end{subfigure}%
\vspace{-3mm}
\caption{Symmetric Setting with $N=50$ and $T=40$. (a) SAFFE-D performs close to Hindsight  achieving high utilization and low $\Delta$Log-NSW, $\Delta\mathbf{A}^\text{mean}$ and $\Delta\mathbf{A}^\text{max}$, (b) SAFFE-D is robust to expected demands estimation errors.}
\vspace{-4mm}
\end{figure*}

\vspace{-5pt}
\subsection{Non-Symmetric Arrivals or Demands}\label{sec:exp-NH}
We investigate whether SAFFE-D allocates to agents differently based on their arrival or demand patterns. We consider $N=50$ agents with $2$ expected arrivals over the horizon $T=40$, and Normal demands with $\mu_i=50$, and budget size $0.5$. Under the Non-symmetric Arrivals setting, we compare the agents' allocations for each  algorithm with respect to the Hindsight allocations in Fig.~\ref{fig:NH-arrivals}. In this setting, agents have different likelihoods of arriving over the horizon. We observe that despite being optimized for the log-NSW objective, SAFFE-D outperforms all other algorithms on average, and is more uniform across the three  groups compared to SAFFE and Guarded-HOPE which favor the Ask-Late agents. In terms of the worst-case agent, SAFFE outperforms all other methods and has similar  $\Delta\mathbf{A}^{max}$ across all three groups. The algorithms are compared under the Non-symmetric Demands setting in Appendix~\ref{app:exp-NH}.

\begin{figure}[!h]
\centering
 \includegraphics[width=\linewidth]{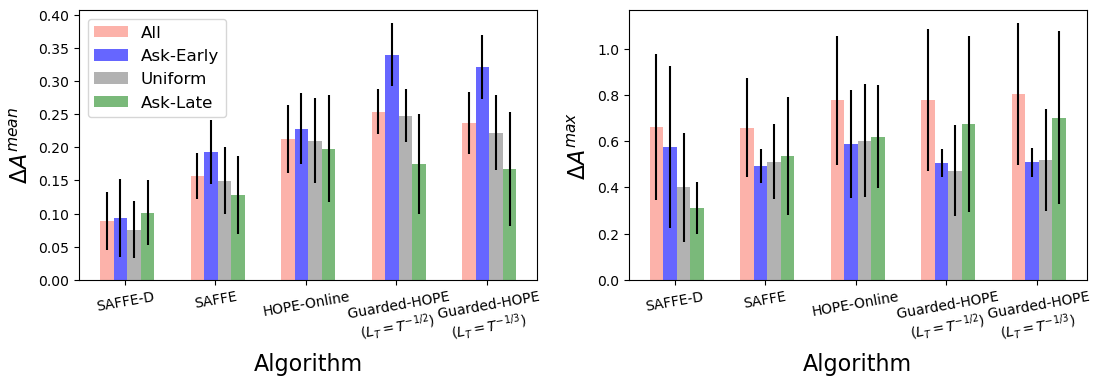}
 \vspace{-6mm}
 \caption{Non-symmetric Arrivals Setting: SAFFE-D outperforms in the average case across groups with different arrival patterns and SAFFE is more uniform in the worst case.}
\label{fig:NH-arrivals}
\vspace{-6mm}
\end{figure}

\begin{table}[!htbp]
\centering
\ra{1.2}
\caption{Improvement of SAFFE-D with choice of $\lambda$.}
\vspace{-3mm}
\resizebox{\linewidth}{!}{%
\begin{tabular}{@{}rrccc@{}}\toprule
& Log-NSW ($\uparrow$)  & Utilization \% ($\uparrow$) & $\Delta \mathbf{A}^\text{mean}$ ($\downarrow$) & $\Delta \mathbf{A}^\text{max}$ ($\downarrow$) \\ \midrule
Hindsight & {236.57$\pm$10.21} & {100.0$\pm$0.0}  & $-$ & $-$ \\
SAFFE-D ($\lambda(t)$) & \textbf{235.91$\pm$10.26}   & \textbf{99.45$\pm$1.5} & \textbf{0.05$\pm$0.02} & \textbf{0.45$\pm$0.23} \\ 
SAFFE-D ($\lambda$) & {234.12$\pm$10.67}   & {97.86$\pm$3.5}  & {0.09$\pm$0.03} & {0.65$\pm$0.22}\\
SAFFE ($\lambda=0$) & {232.89$\pm$10.91}   & {95.77$\pm$4.5}  & {0.11$\pm$0.03}  & {0.69$\pm$0.21}\\
\bottomrule
\end{tabular}
}
\label{tb:lambda}
\vspace{-5mm}
\end{table}

\subsection{Sensitivity to Estimation Errors}\label{sec:exp-sensitivty}
We investigate how sensitive the algorithms are to estimation errors 
by scaling each agent's mean estimate at each  step as $\widehat{\mu}_i^t= (1+\delta)\mu_i^t$, where $\delta$ denotes the noise level. Fig.~\ref{fig:sensitivity} shows  the change in performance  as we vary the noise level $\delta$ from $-0.5$ to $0.5$, for the Symmetric setting with $N=50$ agents, $2$ expected per-agent arrivals,   $T=40$, and budget  $0.5$. We observe that  SAFFE-D is robust to mean estimation noise while all other algorithms are impacted. This results from the  discounting term in SAFFE-D, for which the hyper-parameter $\lambda$ can be tuned to the estimation errors. However, when over-estimating the future demands, SAFFE and other methods will be initially more conservative which can lead to lower utilization and less fair allocations due to the leftover budget. When under-estimating the expected demands, they use the budget more greedily earlier on and  deplete the budget sooner, resulting in less allocations to agents arriving later (especially in Gaurded-HOPE).

\subsection{Choice of $\lambda$ for SAFFE-D}\label{sec:exp-lambda}
The discounting term in SAFFE-D uses the confidence in estimates to improve the performance of SAFFE by balancing the budget between allocating now vs reserving for future arrivals. In Table~\ref{tb:lambda}, we compare the performance of SAFFE-D when using different discounting functions, i.e., constant or decreasing $\lambda$ over the horizon. In both cases, the parameter is tuned to maximize  Log-NSW. As expected from Sec.~\ref{sec:saffe-D},  the performance of SAFFE-D improves for $\lambda(t)$ decreasing over time. This is because the uncertainty in expected future demands reduces and the supplier can be less conservative when making allocation decisions.

\begin{table*}[!h]
\centering
\ra{1.15}
\caption{SAFFE-D vs RL allocation policy in the Symmetric setting for $N=10$ agents with different per-agent arrivals.}
\vspace{-2mm}
\resizebox{0.9\textwidth}{!}{%
\begin{tabular}{@{}rrrccccccc@{}}\toprule
& \multicolumn{4}{c}{$2$ Per-agent Arrivals (Sparse)} & \phantom{abc}& \multicolumn{4}{c}{{$4$ Per-agent Arrivals (Dense)}}  \\
\cmidrule{2-5} \cmidrule{7-10}  
& Log-NSW ($\uparrow$)  & Utilization \% ($\uparrow$) & $\Delta \mathbf{A}^\text{mean}$ ($\downarrow$) & $\Delta \mathbf{A}^\text{max}$ ($\downarrow$) && Log-NSW ($\uparrow$)  & Utilization \% ($\uparrow$) & $\Delta \mathbf{A}^\text{mean}$ ($\downarrow$) & $\Delta \mathbf{A}^\text{max}$ ($\downarrow$) \\ \midrule
Hindsight 
& 35.74$\pm$1.23            & 100.0$\pm$0.0 & $-$                    & $-$ 
&& 47.37$\pm$2.20           & 100.0$\pm$0.0         & $-$                 & $-$  \\
RL Policy (SAC)
& \textbf{35.11$\pm$1.20}   & 95.63$\pm$0.49 & \textbf{0.13 $\pm$ 0.01} & \textbf{0.43 $\pm$ 0.01}
&& 47.00$\pm$2.20           & 98.30$\pm$0.45 &  0.12$\pm$0.01           & 0.34$\pm$0.02      \\
SAFFE-D 
& 35.01$\pm$1.28            & \textbf{99.54$\pm$0.42}  & 0.15$\pm$0.05 & 0.51$\pm$ 0.16       
&& \textbf{47.12$\pm$ 2.22} & \textbf{99.82$\pm$0.36}  & \textbf{0.11$\pm$0.04} & \textbf{0.32$\pm$0.09}     \\
SAFFE    
& 34.77$\pm$1.19            & 92.74$\pm$0.41           & 0.16$\pm$0.02     &  0.53$\pm$0.03     
&& 46.87$\pm$2.15           & 97.14$\pm$0.43           & 0.12$\pm$0.02     & 0.35$\pm$0.03      \\
\bottomrule
\end{tabular}
}
\label{tb:rl}
\vspace{-1mm}
\end{table*}

\begin{table*}[!htbp]
\centering
\ra{1.15}
\caption{Experiments on real data for daily per-agent demands ($p=1$), and more sparse demand arrivals ($p=0.5$).}
\vspace{-2mm}
\resizebox{0.95\textwidth}{!}{%
\begin{tabular}{@{}rrrccccccc@{}}\toprule
& \multicolumn{4}{c}{$p=1$} & \phantom{abc}& \multicolumn{4}{c}{$p=0.5$}  \\
\cmidrule{2-5} \cmidrule{7-10}  
& Log-NSW ($\uparrow$)  & Utilization \% ($\uparrow$) & $\Delta \mathbf{A}^\text{mean}$ ($\downarrow$) & $\Delta \mathbf{A}^\text{max}$ ($\downarrow$) && Log-NSW ($\uparrow$)  & Utilization \% ($\uparrow$) & $\Delta \mathbf{A}^\text{mean}$ ($\downarrow$) & $\Delta \mathbf{A}^\text{max}$ ($\downarrow$) \\ \midrule
Hindsight 
&{42.72$\pm$0.00} & {100.0$\pm$0.0}  & $-$ & $-$ 
&& {35.24$\pm$0.94} & {100.0$\pm$0.0}  & $-$ & $-$  \\
SAFFE-D 
&\textbf{42.72$\pm$0.00} & \textbf{100.0$\pm$0.0}  & \textbf{0} &  \textbf{0} 
&&  \textbf{35.16$\pm$0.94} & \textbf{100.0$\pm$0.0}  & \textbf{0.06$\pm$0.04} &  \textbf{0.17$\pm$0.14} \\
SAFFE  
&{42.72$\pm$0.02} & {99.98$\pm$0.15}   & {0.01$\pm$0.01}   & {0.01$\pm$0.03} 
&& {34.82$\pm$1.10} & {99.16$\pm$2.94}   & {0.17$\pm$0.08}  & {0.37$\pm$0.16}   \\
HOPE-Online  
&{42.71$\pm$0.02} & {99.98$\pm$0.15}   & {0.10$\pm$0.04}   & {0.04$\pm$0.01}
&& {34.44$\pm$1.10} & {99.10$\pm$3.03}    & {0.27$\pm$0.06}   & {0.60$\pm$0.15}\\
Guarded-HOPE ($L_T = T^{-1/2}$) 
&{42.48$\pm$0.30} & {97.93$\pm$2.80}   & {0.07$\pm$0.02}   & {0.16$\pm$0.05}
&& {33.55$\pm$3.67} & {95.93$\pm$6.24}   & {0.27$\pm$0.06}   & {0.62$\pm$0.15}\\
Guarded-HOPE ($L_T = T^{-1/3}$) 
&{42.55$\pm$0.21} & { 98.64$\pm$1.94}   & {0.07$\pm$0.02}   & {0.15$\pm$0.05}
 && {33.59$\pm$3.67} & { 96.54$\pm$5.61}   & {0.28$\pm$0.07}   & {0.63$\pm$0.15}\\
\bottomrule
\end{tabular}
}
\label{tb:realdata}
\vspace{-3mm}
\end{table*}

\subsection{Reinforcement Learning Results}\label{sec:exp-RL}
We investigate 
learning an allocation strategy using RL
under the Symmetric setting. We train the RL allocation policy without access to any demand distribution information (e.g., expectation). We consider $N=10$ agents, $T=10$ and varying budget size sampled from  $\text{Uniform}(0.4,0.8)$. We use the Soft Actor-Critic (SAC)~\cite{HaarnojaZAL18,SpinningUp2018} method with automatic entropy tuning~\cite{HaarnojaZAL18b} to learn a policy (See  Appendix~\ref{app:rldetails}). The  results averaged over $200$ experiment rollouts for five random seeds are reported in Table~\ref{tb:rl}, for two cases where we  expect to have $2$ and $4$ per-agent arrivals. We observe that while the RL policy is not able to match the hindsight performance, it  outperforms SAFFE-D and other approaches in terms of Log-NSW under sparse ($2$ per-agent) arrivals. With denser ($4$ per-agent) arrivals, the RL policy  performs  slightly worse than SAFFE-D while still outperforming others. These observations are in line with other experiments illustrating the close-to-optimal performance of SAFFE-D especially in  settings with more arrivals.



\vspace{-5pt}
\subsection{Real Data}\label{sec:exp-real}
In contrast to the demand processes considered in previous experiments, in this case, we need to estimate the expected demands and their standard deviations. The efficiency and fairness metrics are reported in Table~\ref{tb:realdata} for budget size $0.5$. We observe that SAFFE-D is optimal, and SAFFE and HOPE-Online perform very close to Hindsight in terms of Utilization and Log-NSW. As shown in Fig.~\ref{fig:scale-pfactor} in Appendix~\ref{app:exp-scaling}, we expect that SAFFE-D achieves (close to) hindsight performance for dense agent arrivals as with the daily arrivals in this setting. In order to investigate less dense scenarios, we impose sparsity in the arrivals by erasing each arrival with probability $p$. As observed in Table~\ref{tb:realdata}, SAFFE-D is no  longer optimal but outperforms all other methods. Additional results are provided in Appendix~\ref{app:exp-real}.

%% file: 2_related_work.tex
\section{Related Work} \label{sec:related_work}
Fair division is a classic problem extensively studied for settings involving single and multiple resources, divisible and indivisible resources, and under different fairness objectives \cite{bretthauer1995nonlinear,katoh1998resource,lee2005hybrid,patriksson2008survey}. In the online setting, which requires real-time  decision-making, \cite{walsh2011online,kash2014no} consider the \textit{cake cutting} problem where agents interested in fixed divisible resources arrive and depart over time, while the goal in \cite{aleksandrov2015online,zeng2020fairness,he2019achieving} is to allocate sequentially arriving indivisible resources to a fixed set of agents. A wide range of fairness metrics are used in the literature, including global fairness (e.g., utilitarian or egalitarian welfare) \cite{kalinowski2013social,manshadi2021fair}, individual  measures \cite{benade2018make,gupta2021individual,sinclair2022sequential} or probabilistic notions of fairness across groups of agents \cite{donahue2020fairness}. 


The most relevant work to ours are \cite{lien2014sequential,sinclair2020sequential,sinclair2022sequential}, which study the online allocation of divisible resources to agents with stochastic demands. \cite{lien2014sequential} proposes equitable  policies that maximize the minimum agent utility, while \cite{sinclair2020sequential,sinclair2022sequential} use approximate individual fairness notions. 
Motivated by the food-bank allocation problem, \cite{sinclair2020sequential,sinclair2022sequential} assume one request per agent, where the requests are restricted to a finite set of values. The algorithm proposed in \cite{sinclair2022sequential} improves upon \cite{sinclair2020sequential} and achieves a trade-off between a measure of fairness (envy-freeness) and resource leftover using  concentration arguments on the optimal NSW
solution. 
In our work, we study a more general setup where agents can arrive simultaneously and several times over the horizon and we do not restrict agent demand sizes. Our goal is to be fair to the agents considering all their allocations over the horizon under the NSW objective in hindsight.

%% file: 6_conclusions.tex
 \section{Conclusions}\label{sec:conclusions}
This paper studies the problem of fair resource allocation in the sequential setting, and mathematically formulates the NSW fairness objective under the MDP framework. 
With the goal to maximize the return in the MDP, we propose SAFFE-D, a new  algorithm that enjoys a worst-case theoretical sub-optimality bound. 
Through extensive experiments, we show that SAFFE-D significantly outperforms existing approaches, and effectively balances the fraction of budget that should be used at each step versus preserved for the future. Empirical results under various demand processes demonstrate the superiority of SAFFE-D for different budget sizes, number of agents, and time horizons, especially in settings with dense arrivals. The uncertainty-based discount in SAFFE-D also improves the robustness of  allocations to errors in future demand estimations. Future work includes exploring learning-based allocation policies, learning the optimal $\lambda(t)$ in SAFFE-D, and extensions to settings with strategic agents.


%% file: 7_appendix.tex
\clearpage
\appendix

\twocolumn[
  \begin{@twocolumnfalse}
    \vbox{%
    \hsize\textwidth
    \linewidth\hsize
    \hrule height 4pt
    \vskip 0.25in
    \centering
    {
    \huge\bf
    {Sequential Fair Resource Allocation under a Markov Decision Process Framework \\ 
    Supplementary Materials} \par}
    \vskip 0.2in
    \hrule height 1pt
    \vskip 0.2in
    }
  \end{@twocolumnfalse}
]

\section{Multiple-Resource Setting}\label{app:multi-resource}
The setting in this paper can be easily extended to the supplier having $M$ divisible resources, where resource $j\in\{1,\dots,M\}$ has a limited budget of size $B_j$, and $\mathbf{B}=(B_1,\dots,B_M)$. At time $t\in \{1,\dots,T\}$, agent $i$'s demands i denoted by $\Y^t_{i}=(\Y^t_{i,1}, \dots, \Y^t_{i,M})\in \mathbb{R}^M$, and the supplier's allocation is $\A^t_{i}=(\A^t_{i,1}, \dots, \A^t_{i,M})\in \mathbb{R}^M$. For multiple resources,  $\Theta_i=(\Y_{i,j}^1,\dots,\Y_{i,j}^T, V_i)$, where $V_i=(v_{i,1},\dots,v_{i,M})$ denotes the per-unit value that agent $i$ associates to each resource, 
\begin{align}
    u(\mathbf{\A}_i,\mathbf{\Theta}_i)= \sum_{t=1}^T\sum_{j=1}^M v_{i,j}\min\left\{ \A_{i,j}^t, \,  \Y_{i,j}^t\right\}\label{eq:utility-multi}
\end{align}
While we present the paper in terms of a single resource, all algorithms and proofs can be extend to the multiple-resource setting. In this case, the optimal allocation in hindsight or the sequential setting will not have closed-form solutions as in the water-filling algorithm, but can be computed using convex optimization techniques.

\section{Existence of Optimal Policy} \label{app:existance}
 
The state space of the MDP is continuous, the action space is continuous and state-dependent, and the reward is bounded and continuous due to the constant $\epsilon$ in Eq.~\eqref{eq:def_U}. As a result of \cite{furukawa1972markovian}[Theorem 4.2], a stationary optimal policy $\pi^{\star}=(\pi_1^{\star},\dots,\pi_T^{\star})$ is known to exist.

If we restrict the action space to allocations that satisfy $A_i^t\leq X_i^t$ for all $i$ and $t$, the  reward at step $t$ \eqref{eq:step-reward} becomes
\begin{align}
     &R_t(s_t, \mathbf{A}^t) =\notag \\
     &\quad \sum_{i=1}^N \mathbbm{1} \{X_i^t >0\} \,.\, w_i\Big(\log(\widetilde{\A}_i^{t-1}+\A_i^t+\epsilon) -\log(\widetilde{\A}_i^{t-1}+\epsilon)\Big),\label{eq:reward-detailed}
\end{align}
where $\widetilde{A}_i^t \coloneqq \sum_{\tau=1}^t A_i^\tau$ denotes the cumulative allocated resources to agent $i$ till time $t$. At time step $T$, $\pi_T^\star$  maximizes $Q_T(s_T,\mathbf{A}^T)=R_T(s_T, \mathbf{A}^T)$ which does not depend on future allocations. The optimal allocation ${\mathbf{A}^T}^\star$ can be directly computed from $Q_T$ since there is no uncertainty about future demands. Then, the optimal policy for all $t$ can be derived by recursively solving \eqref{eq:optalg_timet} backward in time. This quickly becomes computationally intractable, which motivates us to find alternative solutions that trade-off sub-optimality for computational efficiency. 

\section{Water-Filling Algorithm with Past Allocations}\label{app:waterfilling-past alloc} 
Algorithm~\ref{alg:waterfilling past alloc} extends the water-filling algorithm presented in Algorithm~\ref{alg:waterfilling-base} for agents with different weights to a setting where each agent may have past allocations denoted by $\mathbf{A}^0$. This algorithm is used as a base algorithm in SAFFE (Algorithm~\ref{alg:SAFFE}) to compute the allocations during each time step based on past allocations, and current  and expected future demands. In line 3, the agents are ordered according to their demands, past allocations and weights. This ordering determines which agent is used to compute the water-level $\mu$. For each selected agent $i_j$, the condition in line 6 determines whether there is enough budget to fully satisfy agent $i_j$'s demand $X_{i_j}$. If there is enough budget, agent $i_j$ receives its full request (line 10) and the supplier moves on to the next agent in order. Otherwise, the available budget is fully divided among the remaining agents (line 8) according to a water-level $\mu$ computed  in line 7. The water-level accounts for the agent's past allocations as well as  their weights.  
\begin{algorithm}
\caption{Water-Filling Algorithm with Past Allocations}
\label{alg:waterfilling past alloc}
\begin{algorithmic}[1]
\STATE{Input: number of agents $N$, resource budget $B$, demand vector ${\mathbf{X}}\in\mathbb{R}^N$, weight vector $\mathbf{w}\in\mathbb{R}^N$, past allocations $\mathbf{A}^{0}\in\mathbb{R}^N$}
\STATE{Output: allocation vector ${\mathbf{A}}\in\mathbb{R}^N$}
\STATE{Find an ordered index set $\{i_1,\dots,i_N\}$ by sorting the agents such that $\frac{{X}_{i_1}+{A}^0_{i_1}}{w_{i_1}}\leq\dots\leq \frac{{X}_{i_N}+{A}^0_{i_N}}{w_{i_N}}$}
\STATE{$j \leftarrow 1$}
\WHILE{$j\leq N$ and $B > 0$}
\IF{$B \leq \sum_{k=j}^{N} \Big(\frac{w_{i_k}}{w_{i_j}}({X}_{i_j}+A_{i_j}^0)-{A}^0_{i_k}\Big)^+$}
\STATE{Solve $\sum_{k=j}^{N}(\frac{w_{i_k}}{w_{i_j}}\mu-A_{i_j}^0)^+=B$ for $\mu$}
\STATE{$A_{i_k}=(\frac{w_{i_k}}{w_{i_j}}\mu-A_{i_j}^0)^+$ for  $k=j,\dots,N$}
 \STATE{$\textbf{break}$}
\ELSE
  \STATE{${A}_{i_j}={X}_{i_j}$}  
  \STATE{$B \leftarrow B - {A}_{i_j}$}
  \STATE{$j\leftarrow j+1$} 
\ENDIF
\ENDWHILE
\end{algorithmic}
\end{algorithm}

\section{SAFFE-Discounted vs SAFFE-Oracle}\label{app:SAFFE-D Vs SAFFE-O}
Under mild assumptions on the distribution of demands, we can quantify the distance between SAFFE-D and SAFFE-Oracle allocations using concentration inequalities on the deviation of future demands from their expected value. Let us define 
\begin{align}
    &\overline{Y}_i^t = X_i^t + \mathbb{E}\Big[\sum_{\tau=t+1}^{T}X_i^{\tau}\Big]+\frac{\sqrt{T-t}}{\sqrt{\xi}}\operatorname{std}(X^{\tau}_{i})\\
    &\underbar{Y}_i^t = X_i^t + \mathbb{E}\Big[\sum_{\tau=t+1}^{T}X_i^{\tau}\Big]-\frac{\sqrt{T-t}}{\sqrt{\xi}}\operatorname{std}(X^{\tau}_{i})
\end{align}
For simplicity assume that agent $i$'s demands are i.i.d (this assumption can be relaxed see Remark \ref{remark: relax same std assumption by agents}). Then, for $\xi>0$, with probability at least 
$1-\xi$, based on
Chebyshev’s inequality we have
\begin{align}
    \underbar{Y}_i^t \leq X^{t}_{i}+\sum_{\tau=t+1}^{T} X^{\tau}_i \leq \overline{Y}_i^t\label{eq_requests_upper_lower}
\end{align}

We further assume that all agents have equal $\operatorname{std}(X^{\tau}_{i})$, but their expectations might differ (the assumption simplifies the presentation, and is not crucial). We first present the worst-case scenario bound, where we allow highly unbalanced future demands for different agents, i.e. there is an agent $k$ such that $\underbar{Y}^{t}_{k} \geq \overline{Y}^{t}_{j}$ for all other agents $j\neq k$. 

\begin{theorem}[Unbalanced demands bound]\label{thm:bound oracle vs saffe}
Let $A_{i}^{t, \text{SAFFE-D}}$ and $A_{i}^{t,\text{oracle}}$ denote allocations by SAFFE-D for $\lambda(t) = \sqrt{\frac{T-t}{{\xi}}}$ and SAFFE-Oracle, respectively. Then, for all agents $i$ we have
\begin{align}\notag
    \left\vert A_{i}^{t,\text{SAFFE-D}} - A_{i}^{t, \text{oracle}} \right\vert \leq 
\begin{cases}
2N\sqrt{\frac{(T-t)}{{\xi}}} \operatorname{std}(X^{t}_{i})  & \text{if }   B^{t}\leq   \sum\limits_{i=1}^{N}\overline{Y}_i^t,\\
4\sqrt{\frac{(T-t)}{{\xi}}} \operatorname{std}(X^{t}_{i})  & \text{if }     B^{t}\geq   \sum\limits_{i=1}^{N}\overline{Y}_i^t
\end{cases}
\end{align}
with probability at least $1-\xi$.
\end{theorem}

\begin{proof}
The detailed proof is in Appendix \ref{app:proof of SAFFE-D Vs SAFFE-O}. Intuitively, the discrepancy scales with the number of agents, since if all agnets $j\neq k$ submit demands according to their upper bound $\overline{Y}^{t}_{j}$, then the water-level moves so that their SAFFE-Oracle allocations increase compared to SAFFE-Oracle with $\underbar{Y}^{t}_{j}$, which happens on the account of agent $j$ who now gets less. Finally, using \eqref{eq_requests_upper_lower} completes the proof, as we can translate the discrepancy between SAFFE-Oracle with $\underbar{Y}^{t}_{i}$ and $\overline{Y}^{t}_{i}$ to that between SAFFE-D and SAFFE-Oracle with $Y^{t}_{i}$.
\end{proof}

\begin{remark} \label{remark: relax same std assumption by agents}
The equal standard deviation assumption that for any two agents $i$ and $j$, we have $\operatorname{std}(X^{\tau}_{i}) = \operatorname{std}(X^{\tau}_{j})$ is used in the proofs only to simplify $\sum_{j\neq i} \operatorname{std}(X^{\tau}_{j})=(N-1)\operatorname{std}(X^{\tau}_{j})$. Thus, the assumption can be removed.
\end{remark}

In order to move beyond the worst-case scenario of highly unbalanced request distributions, we now assume that all agents have the same $\overline{Y}_i^t$ and $\underbar{Y}_i^t$, e.g. if their demands are from the same distribution, the discrepancy between allocations scales better.
\begin{theorem}[Balanced demands bound]\label{thm:special bound oracle vs saffe}
If for any two agents $i$ and $j$ we have the same bounds in (\ref{eq_requests_upper_lower}), i.e. $\overline{Y}_i^t = \overline{Y}_j^t$ and $\underbar{Y}_i^t = \underbar{Y}_j^t$, then for all agents $i$,  with probability at least $1-\xi$, we have
\begin{align}\notag
    \left\vert A_{i}^{t,\text{SAFFE-D}} - A_{i}^{t, \text{oracle}} \right\vert \leq 
\begin{cases}
0 & \text{if } B^{t}\leq   \sum\limits_{i=1}^{N}\underbar{Y}_i^t,\\
4\sqrt{\frac{(T-t)}{{\xi}}} \operatorname{std}(X^{t}_{i})  & \text{if }   B^{t} \geq \sum\limits_{i=1}^{N}\underbar{Y}_i^t 
\end{cases}
\end{align}
\end{theorem}

\begin{proof}
See Appendix \ref{app:proof of SAFFE-D Vs SAFFE-O} for the detailed proof, which follows a similar argument as the proof of Theorem \ref{thm:bound oracle vs saffe}. 
\end{proof}

\section{Detailed Proof of Theorem~\ref{thm:bound oracle vs saffe}} \label{app:proof of SAFFE-D Vs SAFFE-O}
\begin{lemma}[Chebyshev's inequality]
For a random variable $Y$ with finite expectation and a finite non-zero variance, we have
\begin{align*}
    \mathbb{P}\left(\vert Y-\mathbb{E}(Y)\vert \geq \frac{1}{\sqrt{\xi}}\sqrt{\text{Var}(Y)}\right) \leq \xi
\end{align*}
\end{lemma}
In particular, under the assumption from Appendix \ref{app:SAFFE-D Vs SAFFE-O} that agent $i$'s requests are i.i.d, Chebyshev's inequality gives us that Eq.~\eqref{eq_requests_upper_lower} holds with probability $1-\xi$. Note however that the assumption is not crucial, and in a more general case we would have 
$$\underbar{Y}^{t}_{i} = \mathbb{E}[\sum_{\tau=t+1}^{T}X^{\tau}_{i}]-\frac{1}{\sqrt{\xi}}\sqrt{\text{Var}\Big(\sum_{\tau=t+1}^{T}X^{\tau}_{i}\Big)}$$
and
$$\overline{Y}^{t}_{i}=\mathbb{E}[\sum_{\tau=t+1}^{T}X^{\tau}_{i}]+\frac{1}{\sqrt{\xi}}\sqrt{\text{Var}\Big(\sum_{\tau=t+1}^{T}X^{\tau}_{i}\Big)}.$$
See also Remark \ref{remark: relax same std assumption by agents} on how the assumption on equal standard deviations $\operatorname{std}(X^{\tau}_{i}) = \operatorname{std}(X^{\tau}_{j})$ for any two agents $i$ and $j$ is not crucial for the analysis.

We proceed as follows. First, we state the result for SAFFE-Oracle in the unbalanced demands regime, i.e. when there is an agent $k$ such that $\underbar{Y}^{t}_{k}\geq\overline{Y}^{t}_{j}$ for all $j\neq k$. Then, we state the result for the case of balanced demands, i.e. when $\underbar{Y}^{t}_{i}=\underbar{Y}^{t}_{j}$ and $\overline{Y}^{t}_{i}=\overline{Y}^{t}_{j}$ for any two $i$ and $j$. Finally, we present the  proof for both results.

\begin{lemma}[Unbalanced demands SAFFE-Oracle]\label{lem bound alloc confidence interval}
Let $\overline{A}_i^t$ and $\underbar{A}_i^t$ denote the allocations of SAFFE-Oracle with $Y_{i}^{t} = \overline{Y}_i^t$ and $Y_{i}^{t} = \underbar{Y}_i^t$, respectively. Then for all agents $i$ we have
\begin{align}
|\overline{A}_i^t - \underbar{A}_i^t | \leq 
\begin{cases}
2N\sqrt{\frac{ (T-t)}{{\xi}}} \operatorname{std}(X^{\tau}_{i}) & \text{ if } B^{t}\leq   \sum\limits_{i=1}^{N}\underbar{Y}_i^t,\\
2N\sqrt{\frac{ (T-t)}{{\xi}}} \operatorname{std}(X^{\tau}_{i})  & \text{ if }   \sum\limits_{i=1}^{N}\underbar{Y}_i^t \leq B^{t}\leq   \sum\limits_{i=1}^{N}\overline{Y}_i^t,\\
4\sqrt{\frac{ (T-t)}{{\xi}}} \operatorname{std}(X^{\tau}_{i})  & \text{ if }     B^{t}\geq   \sum\limits_{i=1}^{N}\overline{Y}_i^t
\end{cases}
\end{align}
\end{lemma}

\begin{lemma}[Balanced demands SAFFE-Oracle]\label{lem bound alloc confidence interval special case}
Let $\overline{A}_i^t$ and $\underbar{A}_i^t$ denote the allocations of SAFFE-Oracle with $Y_{i}^{t} = \overline{Y}_i^t$ and $Y_{i}^{t} = \underbar{Y}_i^t$, respectively. Then we have
\begin{align}
|\overline{A}_i^t - \underbar{A}_i^t | \leq 
\begin{cases}
0 & \text{ if } B^{t}\leq   \sum\limits_{i=1}^{N}\underbar{Y}_i^t,\\
4\sqrt{\frac{ (T-t)}{{\xi}}} \operatorname{std}(X^{\tau}_{i})  & \text{ if }   \sum\limits_{i=1}^{N}\underbar{Y}_i^t \leq B^{t}\leq   \sum\limits_{i=1}^{N}\overline{Y}_i^t,\\
4\sqrt{\frac{ (T-t)}{{\xi}}} \operatorname{std}(X^{\tau}_{i})  & \text{ if }     B^{t}\geq   \sum\limits_{i=1}^{N}\overline{Y}_i^t
\end{cases}
\end{align}
\end{lemma}

\begin{proof}[Proof of Lemma \ref{lem bound alloc confidence interval} and Lemma \ref{lem bound alloc confidence interval special case}]

~\paragraph{\textbf{Case} \boldmath{$B^{t}\leq   \sum_{n=1}^{N}\underbar{Y}_i^t$.} } \unboldmath Agents whose lower bound demands are fully filled by the lower bound solution with $Y_{i}^{t} = \underbar{Y}_i^t$ might
increase the water-level when submitting larger demands, which will cause other agents to have smaller $\overline{C}^{t}_{j}$ in comparison to $\underbar{C}^{t}_{j}$. In the worst-case scenario, where agent $i$'s lower bound demand is greater than all other agents' upper bound demands, 
the difference $\underbar{C}^{t}_{i} -\overline{C}^{t}_{i}$ can be at most 
$$\sum_{j\neq i}\left(\overline{Y}_i^t-\underbar{Y}_i^t \right),$$
i.e., agent $i$ has to compensate for the excess demands of all other agents. If $\operatorname{std}(X^{\tau}_{i})$ is equal for all agents, the difference in the worst-case scenario is at most $2(N-1)\sqrt{\frac{(T-t)}{{\xi}}}\operatorname{std}(X^{\tau}_{i})$, which leads to a difference between $\underbar{A}^{t}_{i} -\overline{A}^{t}_{i}$ of at most $2N\sqrt{\frac{ (T-t)}{{\xi}}} \operatorname{std}(X^{\tau}_{i})$ as a consequence of Lemma \ref{lem diff between Cs}. 

~ On the other hand if instead of the worst-case scenario we assume that $\overline{Y}_i^t$ and $\underbar{Y}_i^t$ are the same for all agents (e.g. same demand distributions), then the water-level would not move in SAFFE using $Y_{i}^{t} =\overline{Y}_i^t$ in comparison to SAFFE using $Y_{i}^{t} =\underbar{Y}_i^t$, and thus there would be no change in the allocations resulting from the two problems.

~\paragraph{\textbf{Case} \boldmath{ $\sum_{n=1}^{N}\underbar{Y}_i^t\leq B^{t}\leq \sum_{n=1}^{N}\overline{Y}_i^t$.}} \unboldmath There is enough budget to satisfy all agents' demands under the  lower bound constraint but not eoug budget under the upper bound one. 
Compared to $\underbar{C}^{t}_{i}$, in the worst-case scenario,  $\overline{C}^{t}_{i}$ can decrease by at most
$$\sum_{j \neq i}\overline{Y}_i^t - \sum_{j \neq i}\underbar{Y}_i^t $$
 where agent $i$'s lower bound demand is larger than any other agents' upper bound demands, and therefore, the  water-filling algorithm will satisfy the upper bound demands of all other agents on the account of agent $i$. If all agents have equal  $\operatorname{std}(X^{\tau}_{i})$, in the worst-case scenario, the difference  is again at most $2(N-1)\sqrt{\frac{(T-t)}{{\xi}}}\operatorname{std}(X^{\tau}_{i})$. Following  Lemma \ref{lem diff between Cs}, this leads to a difference of at most $2N\sqrt{\frac{ (T-t)}{{\xi}}} \operatorname{std}(X^{\tau}_{i})$ between $\underbar{A}^{t}_{i}$ and $\overline{A}^{t}_{i}$.

~ On the other hand, if we assume that $\sum_{j \neq i}\overline{Y}_i^t$ and $\sum_{j \neq i}\underbar{Y}_i^t$ are the same for all agents, then all agents will receive  an equal share of the surplus, i.e. the difference between $\underbar{C}_i^t$ and $\overline{C}_i^t$  becomes
$$\frac{1}{N}\left({B^{t}}-\sum_{j \neq i}\underbar{Y}_i^t\right).$$
Given that $B^{t}\leq \sum_{n=1}^{N} \overline{Y}_i^t$, and due to \eqref{eq_requests_upper_lower}, this is at most $2 \sqrt{\frac{(T-t)}{{\xi}}}$ $\operatorname{std}(X^{\tau}_{i})$. As a consequence of Lemma \ref{lem diff between Cs}, this leads to difference of at most $4\sqrt{\frac{(T-t)}{{\xi}}}\operatorname{std}(X^{\tau}_{i})$ between $\underbar{A}_i^t$ and $\overline{A}_i^t$.

~\paragraph{\textbf{Case} \boldmath{ $\sum_{n=1}^{N}\overline{Y}_i^t\leq B^{t}$.}} \unboldmath There is enough budget to satisfy all agents' demands even under the upper bound, so the difference between $\overline{C}^{t}_{i}$ and $\underbar{C}^{t}_{i}$ is purely driven by the difference between lower and upper bound demands, i.e. it equals $\overline{Y}_i^t-\underbar{Y}_i^t$ which is $2\sqrt{\frac{(T-t)}{{\xi}}}\operatorname{std}(X^{\tau}_{n})$. Following Lemma \ref{lem diff between Cs}, this leads to a difference of at most $4\sqrt{\frac{(T-t)}{{\xi}}}\operatorname{std}(X^{t'}_{i})$ between
$\underbar{A}^{t}_{i}$ and $\overline{A}^{t}_{i}$.
\end{proof}

\begin{lemma} \label{lem diff between Cs}
If    
$\vert \overline{C}^{t}_{i} - \underbar{C}^{t}_{i} \vert \leq D$, then we have
\begin{align*}
    \vert \overline{A}^{t}_{i} - \underbar{A}^{t}_{i} \vert \leq D+2\sqrt{\frac{(T-t)}{{\xi}}} \operatorname{std}(X^{\tau}_{i}).
\end{align*}
\end{lemma}

\begin{proof}
Firstly, let us note that
\begin{align*}
    X_{i}^{t}\overline{C}^{t}_{i} \underbar{Y}^{t}_{i}-X_{i}^{t}\underbar{C}^{t}_{i} {\overline{Y}^{t}_{i}} &= X_{i}^{t} \underbar{Y}^{t}_{i}(\overline{C}^{t}_{i}-\underbar{C}^{t}_{i})
    + X_{i}^{t}\underbar{C}^{t}_{i}\underbar{Y}^{t}_{i} \\
    & - X_{i}^{t}\underbar{C}^{t}_{i}\underbar{Y}^{t}_{i} - X_{i}^{t}\underbar{C}^{t}_{i}\left( \overline{Y}^{t}_{i} -\underbar{Y}^{t}_{i} \right)
\end{align*}
and hence we have
\begin{align*}
    \left\vert X_{i}^{t}\overline{C}^{t}_{i} \underbar{Y}^{t}_{i}-X_{i}^{t}\underbar{C}^{t}_{i} {\overline{Y}^{t}_{i}} \right\vert &\leq  X_{i}^{t} \underbar{Y}^{t}_{i}\left\vert \overline{C}^{t}_{i}-\underbar{C}^{t}_{i} \right\vert+  X_{i}^{t}\underbar{C}^{t}_{i}\left\vert \overline{Y}^{t}_{i}-\underbar{Y}^{t}_{i} \right\vert\\
    & \leq  X_{i}^{t} \underbar{Y}^{t}_{i} \cdot D + 2X_{i}^{t}\underbar{C}^{t}_{i}\sqrt{\frac{T-t}{{\xi}}}\operatorname{std}(X^{\tau}_{i})\\
    &\leq X_{i}^{t} \underbar{Y}^{t}_{i}\left(D+2\sqrt{\frac{T-t}{{\xi}}}\operatorname{std}(X^{\tau}_{i})\right)
\end{align*}
where in the last inequality we use $\underbar{C}^{t}_{i} \leq \underbar{Y}^{t}_{i}.$

Thus
\begin{align*}
    \vert \overline{A}^{t}_{i} - \underbar{A}^{t}_{i} \vert &= \left\vert \frac{X_{i}^{t}\overline{C}^{t}_{i}}{\overline{Y}^{t}_{i} } - \frac{X_{i}^{t}\underbar{C}^{t}_{i}}{\underbar{Y}^{t}_{i} } \right\vert\\
    &= \left\vert \frac{X_{i}^{t}\overline{C}^{t}_{i} \underbar{Y}^{t}_{i}-X_{i}^{t}\underbar{C}^{t}_{i} \overline{Y}^{t}_{i}}{\overline{Y}^{t}_{i}\underbar{Y}^{t}_{i}}   \right\vert \\
   & \leq \frac{(D+2\sqrt{\frac{ (T-t)}{{\xi}}} \operatorname{std}(X^{\tau}_{i})) X_{i}^{t}}{\overline{Y}^{t}_{i}}\\
   & \leq D+2\sqrt{\frac{ (T-t)}{{\xi}}} \operatorname{std}(X^{\tau}_{i})
\end{align*}
\end{proof}


\section{SAFFE-Oracle Matches Hindsight} \label{app: discrepancy saffe oracle achieves hindsight}
In Appendix~\ref{app:SAFFE-D Vs SAFFE-O}, we have upper bounded the discrepancy between the allocations of SAFFE-D and SAFFE-Oracle. Our next step is to show that SAFFE-Oracle achieves the optimal allocations in hindsight.
\begin{theorem} \label{thm:oracle achieves hindsight}
SAFFE-Oracle achieves optimal allocations in hindsight, i.e. for all $i$ we have $\sum_{t=1}^{T} A^{t, \text{oracle}}_{i} = \widetilde{A_{i}}$, 
where $\widetilde{A_{i}}$ is the solution to the Eisenber-Gale program (\ref{eq:obj_offline0}).
\end{theorem}

\begin{proof}
See Appendix~\ref{app:saffe oracle achieves hindsight}  for a detailed proof, which  relies on the following observations: 1) For $t=1$, total allocations for the current step and the reserved allocations for the future $C^{1}_{i}$, equals the solution $\widetilde{A_{i}}$ of \eqref{eq:obj_offline0}, and 2) SAFFE-Oracle fully distributes the total reserved allocations for future by the end of the horizon $T$.
\end{proof}

\section{Detailed Proof of Theorem~\ref{thm:oracle achieves hindsight}} \label{app:saffe oracle achieves hindsight}
\begin{lemma}\label{lem: saffe oracle EG}
At t=1, the solution $C_{i}^{1, \text{oracle}}$ to the  SAFFE-Oracle problem coincides with the solution $\tilde{A_{i}}$ to the Eisenber-Gale program (\ref{eq:obj_offline0}) i.e. $C_{i}^{1, \text{oracle}}=\widetilde{A_{i}}$.
\end{lemma}

\begin{proof}
At $t=1$, SAFFE-Oracle solves the following problem 

\begin{align*}
    \max_{\{C_i^1\}}&\sum_{i:\sum_{t=1}^{T}X_i^t>0}\log(C_i^1)\\
    \text{s.t. }&0\leq C_i^1\leq \sum_{t=1}^{T}X_i^t,\quad\sum_{n=1}^{N}C_i^1\leq B.\notag
\end{align*}
This is equivalent to problem (\ref{eq:obj_offline0}) with $A_{i} = C_{i}^{1}$, and hence, we have $C_{i}^{1, \text{oracle}}=\tilde{A_{i}}$.
\end{proof}

\begin{lemma}\label{lem: oracle distributes all} With SAFFE-Oracle, for  any agent $i$ we have
$$C^{1, \text{oracle}}_{i} = \sum_{t=1}^{T}{A^{t, \text{oracle}}_{i}}$$
i.e. its future allocation from the first round is allocated in full by the end of the $T^\text{th}$ round.
\end{lemma}

\begin{proof}
First, we prove that for  $t=2, \dots, T$ we have
$$\sum_{\tau=1}^{t-1} A^{\tau, \text{oracle}}_{i} + C^{t, \text{oracle}}_{i} = \sum_{\tau=1}^{t-2} A^{\tau, \text{oracle}}_{i} + C^{t-1, \text{oracle}}_{i}$$
i.e. the past plus future allocation from the previous step is preserved in the following step. Recall that in each step we are solving
\begin{align}
    \max_{\{C_i^t\}}&\sum_{i:\sum_{\tau=t}^{T}X_n^{\tau}>0}\log(\sum_{\tau=1}^{t-1}A^{\tau, \text{oracle}}_{i}+C_i^t)\label{eq:fair_online_discount}\\
    \text{s.t. }&0\leq C_i^t\leq \sum_{\tau=t}^{T}X_i^{\tau},\quad\sum_{n=1}^{N}C_i^t\leq B^t.\notag
\end{align}
where the sequence of future demands is known to the oracle. 
It is easy to see that 
$$\sum_{\tau=1}^{t-1} A^{\tau, \text{oracle}}_{i} + C^{t, \text{oracle}}_{i} \leq \sum_{\tau=1}^{t-2} A^{\tau, \text{oracle}}_{i} + C^{t-1, \text{oracle}}_{i}$$
because the constraints in step $t$ are tighter. If this didn't hold, the  previous round solution 
$C^{t-1, \text{oracle}}_{i}$ would not be the maximizing solution in step $t-1$. 
It remains to note that $\sum_{\tau=1}^{t-1} A^{\tau, \text{oracle}}_{i} + C^{t, \text{oracle}}_{i}$ can achieve its upper bound, because compared to the previous step $t-1$, the remaining budget was reduced by exactly 
$$\sum_{n=1}^{N}\left( \sum_{\tau=1}^{t-1}A^{\tau, \text{oracle}}_{i}-\sum_{\tau=1}^{t-2}A^{\tau, \text{oracle}}_{i} \right).$$
Therefore, 
$$C_i^{t, \text{oracle}} = C_i^{t-1, \text{oracle}} - \left(\sum_{\tau=1}^{t-1}A^{\tau, \text{oracle}}_{i}-\sum_{\tau=1}^{t-2}A^{\tau, \text{oracle}}_{i}\right)$$ satisifes the constraints. This means that $C_i^{t, \text{oracle}} = C_i^{t-1, \text{oracle}} - A^{t-1, \text{oracle}}_{i}$, i.e., the current plus future allocation in step $t$ equals its counterpart in  step $t-1$ minus the allocation in step $t-1$. In particular, we get
\begin{align}
   C^{2, \text{oracle}}_{i} &= C^{1, \text{oracle}}_{i}-A^{1, \text{oracle}}_{i} \\
   & = C^{1, \text{oracle}}_{n} - \frac{X^{1}_{i}}{\sum_{\tau=1}^{T}X^{\tau}_{i}}C^{1, \text{oracle}}_{i} \\
   & = \left(1 -\frac{X^{1}_{n}}{\sum_{\tau=1}^{T}X^{\tau}_{n}} \right)C^{1, \text{oracle}}_{i},\\
   C^{3, \text{oracle}}_{i} &= C^{2, \text{oracle}}_{i}-A^{2, \text{oracle}}_{i} \\
   & = \left(1-\frac{X^{2}_{n}}{\sum_{\tau=2}^{T}X^{\tau}_{n}} \right) C^{2, \text{oracle}}_{i}\\
   &= \left(1-\frac{X^{2}_{n}}{\sum_{\tau=2}^{T}X^{\tau}_{n}} \right)\left(1 -\frac{X^{1}_{n}}{\sum_{\tau=1}^{T}X^{\tau}_{n}} \right)C^{1, \text{oracle}}_{i}
\end{align}
and in general we have
\begin{align}
  C^{t, \text{oracle}}_{i} = \left(1-\frac{X^{t-1}_{i}}{\sum_{\tau=t-1}^{T}X^{\tau}_{i}} \right)\cdots\left(1 -\frac{X^{1}_{i}}{\sum_{\tau=1}^{T}X^{\tau}_{i}} \right)C^{1, \text{oracle}}_{i}
\end{align}

and 

\begin{align}
  &A^{t, \text{oracle}}_{n} = \notag\\
  &\left(1-\frac{X^{t-1}_{i}}{\sum_{\tau=t-1}^{T}X^{\tau}_{i}} \right)\cdots\left(1 -\frac{X^{1}_{i}}{\sum_{\tau=1}^{T}X^{\tau}_{i}} \right)\frac{X^{t}_{i}}{\sum_{\tau=t}^{T}X^{\tau}_{i}}C^{1, \text{oracle}}_{i}
\end{align}
\end{proof}

\begin{proof}[Proof of Theorem \ref{thm:oracle achieves hindsight}]
As a consequence of Lemmas \ref{lem: saffe oracle EG} and \ref{lem: oracle distributes all}, we have $\tilde{A}^{t}_{i} = \sum_{t=1}^{T}{A^{t, \text{oracle}}_{i}}$ for every $i$. This completes the proof.
\end{proof}

\section{Detailed Proof of Theorem \ref{thm:final gap}}\label{app:detailed proof final thm}

\begin{proof}[Proof of Theorem \ref{thm:final gap}]
Firstly, let us note that Theorem \ref{thm:oracle achieves hindsight} guarantees that 
$$\Delta\mathbf{A}^\text{max}=\mathbb{E}\left[ \max_{i} \left\vert \sum_{t=1}^{T} A^{t, \text{SAFFE-D}} - \sum_{t=1}^{T} A^{t, \text{oracle}} \right\vert \right],$$
as SAFFE-Oracle achieves the hindsight solution.
It remains to employ Theorems \ref{thm:bound oracle vs saffe} and \ref{thm:special bound oracle vs saffe} together with the triangle inequality in order to obtain an upper bound on the right hand side. In the case of unbalanced demands, this results in  $\Delta\mathbf{A}^\text{max} \leq N \frac{T^{3/2}}{\sqrt{\xi}} \operatorname{std}(X^{t}_{i})$. In the case of balanced demands, we have $\Delta\mathbf{A}^\text{max} \leq \frac{2T^{3/2}}{\sqrt{\xi}} \operatorname{std}(X^{t}_{i})$
\end{proof}

\begin{remark}
As $\Delta\mathbf{A}^\text{mean}\leq\Delta\mathbf{A}^\text{max}$, Theorem \ref{thm:final gap} also provides an upper bound on $\Delta\mathbf{A}^\text{mean}$.
\end{remark}

\section{SAFFE  as a generalization of HOPE-Online} \label{app:SAFFE-hope}
The heuristic algorithm HOPE-Online in \cite{sinclair2020sequential} coincides with SAFFE under a simplified demand process. HOPE-Online is designed for a setting where the supplier visits a set of agents sequentially in order. SAFFE generalizes the allocation algorithm to a setup where agents can arrive simultaneously and several times over the horizon $T$. Similar to SAFFE-D which improves SAFFE using the uncertainty of future demand estimates, Guarded-HOPE proposed in \cite{sinclair2022sequential} improves HOPE-Online and achieves the correct trade-off between a measure of fairness (envy-freeness) and resource leftover by learning a ``lower guardrail" on the
optimal solution in hindsight. We highlight that Guarded-HOPE is designed for the setting where agents are not making repeated demands over the horizon i.e. each individual agent has a request at most once during the horizon. In that sense, our SAFFE and SAFFE-D represent generalizations as we allow for agents with multiple demands. In our analysis of the optimality of SAFFE-D, we rely on concentration inequalities for deviation of future demands from their expected values, which is in spirit similar to the optimality analysis of the guardrail approach in \cite{sinclair2022sequential}.

\section{Extension of Guarded-HOPE to Our Setting} \label{app:gaurded-hope}
Since Guarded-HOPE was designed for a setting where in each time-step a number of individuals of the same type  arrive, we slightly modify it to be applicable to our setting in Algorithm~\ref{alg:GuardedHOPE}.
The key components of the algorithm are the upper and lower guardrails defined in lines 7 and 10. Specifically, $\underline{\mathbf{X}},\overline{\mathbf{X}}\in\mathbb{R}^N$ are high-confidence lower and upper bounds on the future demands, and $\overline{\mathbf{A}},\underline{\mathbf{A}}\in\mathbb{R}^N$, which \cite{sinclair2020sequential} refers to as upper and lower guardrails, respectively, are the optimal hindsight allocations under $\underline{\mathbf{X}},\overline{\mathbf{X}}$. When the demands $\mathbf{X}^t$ are revealed in time step $t$, the condition in line 13 first checks if the budget is insufficient to even allow an allocation according to the lower guardrail. If so, the budget is allocated right away and none is reserved. The condition in line 15 checks if the current demands can be allocated according to the upper guardrail assuming that the remaining budget is still enough to allow for allocations according to the lower guardrail for anticipated future demands. If so, the the upper guardrail is allocated to the agents with demands at the current step. Otherwise, the lower guardrail is allocated.

\begin{algorithm}
\caption{Guarded-HOPE (Modified Compared to \protect\cite{sinclair2022sequential})}
\label{alg:GuardedHOPE}
\begin{algorithmic}[1]
\STATE{Input: number of agents $N$, resource budget $B$, demand vectors $\mathbf{X}^1,\dots \mathbf{X}^T\in\mathbb{R}^N$,  demand distributions  $P_{\mathbf{X}_1},\dots,P_{\mathbf{X}_N}$, bound on envy $L_T$}
\STATE{Output: allocation vectors $\mathbf{A}^1,\dots,\mathbf{A}^T\in\mathbb{R}^N$}
\STATE{Define confidence bound $\text{CONF}_i = \sqrt{\text{std}(X_i^1) \mathbb{E}[X_i^1](T-1)} $}
\STATE{$\mathbb{E}X_i \coloneqq\sum_{\tau=1}^T \mathbb{E}[X_i^\tau]$}
\STATE{$c_i = L_t(1 +  \text{CONF}^t_i/\mathbb{E}X_i) - \text{CONF}^t_i/ \mathbb{E}X_i$ for all $i=1,\dots,N$}
\STATE{Solve $\overline{\mathbf{A}}$ for $\underline{X}\coloneqq \mathbb{E}X(1-c)$ using Algorithm~\ref{alg:waterfilling-base}}
\STATE{$\overline{A}_i \leftarrow \overline{A}_i/\underline{X}_i$}
\STATE{$\gamma_i=\text{CONF}^t_i/ \mathbb{E}X_i $}
\STATE{Solve $\underline{\mathbf{A}}$ for $\overline{X}\coloneqq \mathbb{E}X(1+\gamma)$ using Algorithm~\ref{alg:waterfilling-base}}
\STATE{$\underline{A}_i \leftarrow \underline{A}_i/\overline{X}_i$}
\FOR{For $t=1,\dots,T$}
\STATE{Define confidence bound $\text{CONF}^t_i = \sqrt{\text{std}(X_i^t) \mathbb{E}[X_i^t](T-t)} $}
\IF{$B \leq \sum_{i=1}^N X_i^t \underline{A}_i$} 
\STATE{$A_i^t =  \mathbbm{1}\{X_i^t>0\}\times B/\Big(\sum_{i=1}^N \mathbbm{1}\{X_i^t>0\}\Big)$ for $i=1,\dots,N$}
\ELSIF{$B \geq \sum_{i=1}^N X_i^t \overline{A}_i + \sum_{i=1}^N\underline{A}_i \Big(\mathbb{E}[\sum_{\tau=t}^T X_i^\tau]+\text{CONF}_i^t\Big)$}
\STATE{$A_i^t = \mathbbm{1}\{X_i^t>0\}\times X_i^t \overline{A}_i $ for $i=1,\dots,N$}
\ELSE
\STATE{$A_i^t = \mathbbm{1}\{X_i^t>0\}\times X_i^t \underline{A}_i $ for $i=1,\dots,N$}
\ENDIF
\STATE{$B \leftarrow B-\sum_{i=1}^NA_i^t$}
\ENDFOR
 \end{algorithmic}
\end{algorithm}

\section{Additional Experiments}\label{app:experiments}

\subsection{Scaling System Parameters}\label{app:exp-scaling}
SAFFE-D is compared to the baselines in terms of utilization and fairness metrics in Fig.~\ref{fig:scaling plots}, as different system parameters vary. Unless otherwise stated, the settings have $N=50$ agents, budget size $0.5$, time horizon $T=40$, and there are $2$ expected arrivals per-agent. In all experiments, we observe that SAFFE-D is more efficient and more fair compared to other methods, i.e., it achieves higher utilization, and lower $\Delta$Log-NSW,  $\Delta\mathbf{A}^\text{mean}$ and $\Delta\mathbf{A}^\text{max}$. In Fig.~\ref{fig:scale-N}, we observe that as the number of agents increases over the same horizon, the algorithms initially achieve higher utilization and lower $\Delta$Log-NSW, which eventually levels out. However, in terms of $\Delta\mathbf{A}^\text{max}$ fairness which measures the  allocation difference with respect to hindsight for the worst agent, SAFEE-D is the least affected across different number of agents, while all other algorithms become worse.  Fig.~\ref{fig:scale-T} shows the metrics for varying horizon $T$ while having the same number of expected arrivals, i.e., when the arrivals are more spread out across time. In this setting, utilization and fairness metrics seem relatively unaffected across all methods.

In Fig.~\ref{fig:scale-pfactor}, we compare the algorithms as the number of arrivals over the horizon increases. We observe that SAFFE-D is able to use all the budget and match the fair hindsight allocations when it is expected to have  more than $5$ arrivals per-agent over $T=40$. SAFFE is able to match the performance of SAFFE-D in terms of utilization and $\Delta$Log-NSW as the arrivals become denser. As discussed in Appendix~\ref{app:SAFFE-hope},  HOPE-Online and Guarded-HOPE are designed for settings with a single per-agent arrival. As expected, when there are several arrivals per agents, they are not able to match SAFFE-D or SAFFE since they do not account for an agent's past allocations when distributing the budget.  


\begin{figure*}[!h]
\centering
\begin{subfigure}{\linewidth} \centering
 \includegraphics[width=\linewidth]{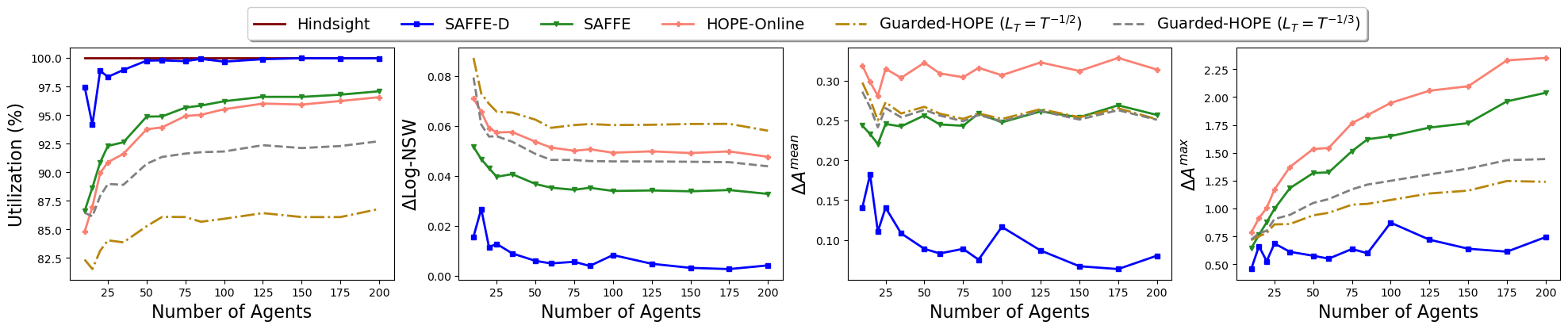}
\caption{Different number of agents $N$.}
\label{fig:scale-N}
\end{subfigure}%
\\
\begin{subfigure}{\linewidth} \centering
\includegraphics[width=\linewidth]{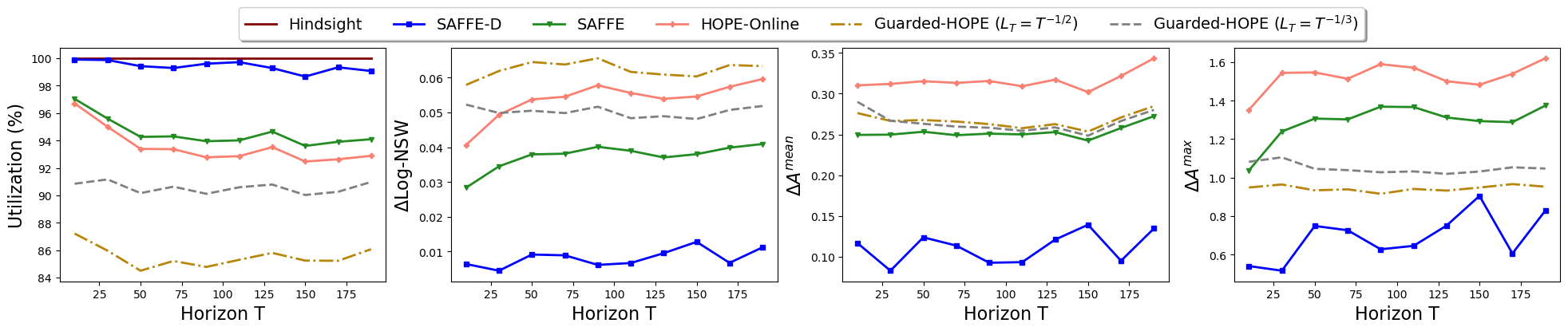}
\caption{Different horizons $T$.}
\label{fig:scale-T}
\end{subfigure}%
\\
\begin{subfigure}{\linewidth} \centering
\includegraphics[width=\linewidth]{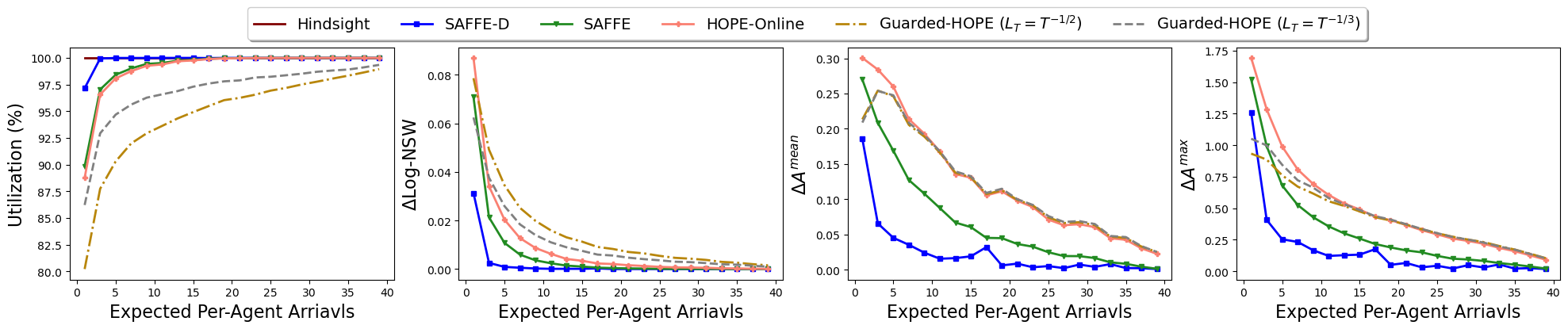}
\caption{Different per-agent expected arrivals.}
\label{fig:scale-pfactor}
\end{subfigure}%

\caption{Symmetric Setting : SAFFE-D performs close to Hindsight and outperforms other methods by achieving higher utilization and lower $\Delta$Log-NSW, $\Delta\mathbf{A}^\text{mean}$ and $\Delta\mathbf{A}^\text{max}$.}
\label{fig:scaling plots}
\end{figure*}

\subsection{Non-Symmetric Demands}\label{app:exp-NH}
Fig.~\ref{fig:NH-demands} shows how SAFFE-D compares to other baselines in terms of allocations using the Non-symmetric Demands Setting, where some agents request larger demands earlier, while others have larger requests later or have no preference.  We observe that  SAFFE-D and SAFFE achieve lower $\Delta\mathbf{A}^\text{mean}$, and have low variability across the groups on average. When considering the worst-case agent, SAFFE-D is less fair to agents that have larger demands earlier in the horizon, as it reserves the budget early on  accounting for future expected demands. However, Uniform and  More-Late agents receive allocations closer to hindsight compared to  other algorithms. 

\begin{figure}[!h]
\centering
\includegraphics[width=\linewidth]{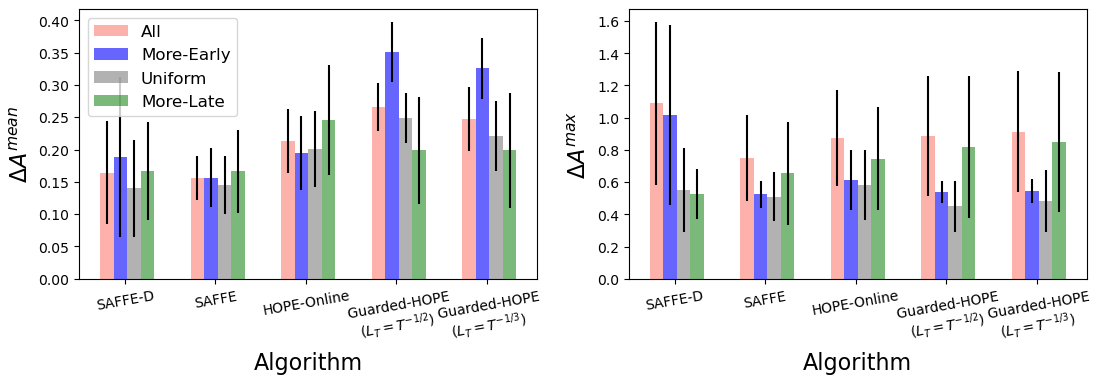}
\caption{Non-symmetric Demands Setting: SAFFE-D allocates uniformly across agents with different arrival patterns on average, while SAFFE outperforms in the worst case.}
\label{fig:NH-demands}
\end{figure}

\subsection{Real Data}\label{app:exp-real}
In order to study how SAFFE-D performs on real data with sparser arrivals over the horizon, we enforce less per-agent requests by erasing each arrival with probability $p$. Since $T=7$, setting $p=\frac{2}{7}$ corresponds to two weekly demands per store. As observed from Fig.~\ref{fig:realdata-B}, SAFFE-D outperforms the other methods in terms of efficiency and fairness. We remark that while with the uniform random erasures, we have imposed Bernoulli arrivals similar to the demand processes described in Sec.~\ref{sec:experiments}, the results presented here are based on the real demand sizes and reflect using imperfect estimates that are computed based on real data. However, further experiments on real datasets are needed to compare the performance of SAFFE-D under more general arrival processes that are correlated across time.

\begin{figure*}[!h]
\centering
 \includegraphics[width=\linewidth]{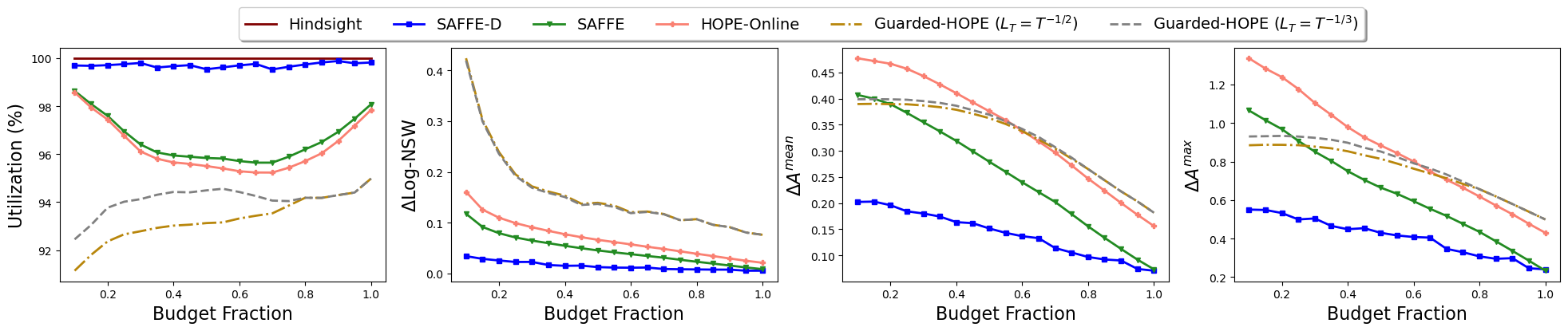}
\caption{Experiments on real data for different budgets when $p=\frac{2}{7}$.  SAFFE-D performs close to Hindsight  achieving high utilization and low $\Delta$Log-NSW, $\Delta\mathbf{A}^\text{mean}$ and $\Delta\mathbf{A}^\text{max}$.}
\label{fig:realdata-B}
\end{figure*}

\section{Training of Reinforcement Learning Policy}
\label{app:rldetails}

We implement a variation of the Soft Actor-Critic method~\cite{HaarnojaZAL18} with clipped double-Q learning~\cite{SpinningUp2018} as well as automatic entropy~\cite{HaarnojaZAL18b} tuning. We train the SAC policy for $100$K episodes for $5$ random seeds, and then evaluate the $5$ checkpoints under another $200$ rollouts in order to compare with other baselines. We report the average performance with one standard deviation over the four metrics of Log-NSW, Utilization, $\Delta \mathbf{A}^\text{mean}$ and $\Delta \mathbf{A}^\text{max}$ in Table~\ref{tb:rl}.

The policy network architecture consists of three Fully Connected (FC) layers followed by one output layer, where each FC layer has 256 neurons with ReLU  activation functions. Since the MDP state is time dependent, to prevent the input state vector size from growing over the time horizon $T$ during training, we represent the state vector as $\langle \widetilde{\mathbf{X}}^{t},\widetilde{\mathbf{A}}^{t}, B^t, t\rangle$, where 
$$ \widetilde{X}_i^t = \sum_{\tau=1}^t {X}_i^{\tau}, \quad \widetilde{A}_i^t = \sum_{\tau=1}^t {A}_i^{\tau}. $$
Note that the state does not include any demand distribution information
(e.g., expectation). We guarantee the step-wise budget constraint $\sum_{i=1}^N A^t_i \leq B^t$, by designing the output layer to have two heads: one head with a Sigmoid activation function, which outputs a scalar determining the step-wise budget utilization ratio $u^t\in[0,1]$; and another head with a Softmax activation function, which outputs an allocation ratio vector, $\mathbf{z}^t\in\mathbb{R}^{N}$ : $\sum_{i=1}^{N}z_i^t=1$ over the agents. The final allocation for each agent is  determined by $A_i^t=z_i^t u^t B^t$.


We perform hyper-parameter tuning using a grid search over a set of candidates. The results in Table~\ref{tb:rl} are achieved by using the hyper-parameters summarized in Table~\ref{tb:hyper-parameter}.

\begin{table}[h!]
    \caption {SAC Hyper-parameters used in experiments.}
    \centering
    \resizebox{\linewidth}{!}{%
    \begin{tabular}{lcc}
    \toprule
        Parameter & Sparse & Dense  \\
          & ($2$ arrivals) & ($4$ arrivals) \\
    \cmidrule(r){2-3}
        Learning rate & $3 \cdot 10^{-4}$ & $3\cdot10^{-4}$ \\
        Replay-buffer size & $8\cdot10^5$ & $10^6$ \\
        Batch size & 512 & 512 \\
        Target smoothing coefficient$(\tau)$ & 0.005 & 0.005\\
        Update interval (step) & 5 & 1\\
        Update after (step) & $10^5$ & $5\cdot10^5$ \\
        Uniform-random action selection (step) &$10^5$ & $5\cdot10^5$ \\
    \bottomrule
    \end{tabular}
    \label{tb:hyper-parameter}
    }
\end{table}